\documentclass{article}

\usepackage{microtype}
\usepackage{graphicx}
\usepackage{subfigure}
\usepackage{booktabs} 

\usepackage{hyperref}

\usepackage[accepted]{icml2024}

\usepackage{amsmath}
\usepackage{amssymb}
\usepackage{mathtools}
\usepackage{amsthm}

\usepackage[capitalize,noabbrev]{cleveref}

\theoremstyle{plain}
\newtheorem{theorem}{Theorem}[section]
\newtheorem{proposition}[theorem]{Proposition}
\newtheorem{lemma}[theorem]{Lemma}
\newtheorem{corollary}[theorem]{Corollary}
\theoremstyle{definition}
\newtheorem{definition}[theorem]{Definition}
\newtheorem{assumption}[theorem]{Assumption}
\theoremstyle{remark}
\newtheorem{remark}[theorem]{Remark}

\usepackage[textsize=tiny]{todonotes}
\usepackage{thm-restate}
\usepackage{thmtools}
\usepackage{cleveref}

\ifx\theorem\undefined
\newtheorem{theorem}{Theorem}

\newtheorem{lemma}[theorem]{Lemma}

\newtheorem{remark}[theorem]{Remark}

\fi

\DeclareMathOperator*{\argmin}{argmin}

\DeclareMathOperator*{\tr}{tr}
\DeclareMathOperator*{\rank}{rank}
\DeclareMathOperator*{\diag}{diag}
\DeclareMathOperator*{\sign}{sign}
\DeclareMathOperator*{\bin}{bin}
\DeclareMathOperator*{\digit}{digit}
\DeclareMathOperator*{\Bin}{Bin}

\def\iid{\overset{\textup{i.i.d.}}{\sim}}

\def\cN{\mathcal{N}}
\def\cL{\mathcal{L}}
\def\cX{\mathcal{X}}
\def\cY{\mathcal{Y}}
\def\cT{\mathcal{T}}

\def\R{\mathbb{R}}
\def\E{\mathbb{E}}

\def\eigen{\mathbf{D}}

\newcommand{\inner}[1]{\left\langle #1 \right \rangle} 
\def\act{\sigma}

\def\A{\mathbf{A}}
\def\X{\mathbf{X}} 
\def\x{\mathbf{x}} 
\def\y{\mathbf{y}} 
\def\W{\mathbf{W}} 
\def\V{\mathbf{V}} 
\def\a{\mathbf{a}} 
\def\G{\mathbf{G}} 
 
\def\t{\mathbf{t}} 
\def\r{\mathbf{r}} 
\def\z{\mathbf{z}} 
\def\Z{\mathbf{Z}} 
\def\v{\mathbf{v}} 
\def\D{\mathbf{D}} 
\def\bE{\mathbf{E}} 
\def\w{\mathbf{w}}
\def\U{\mathbf{U}} 
\def\Q{\mathbf{Q}} 
\def\s{\mathbf{s}}

\icmltitlerunning{Why Larger Language Models Do In-context Learning Differently?}

\begin{document}

\twocolumn[
\icmltitle{Why Larger Language Models Do In-context Learning Differently?}



\icmlsetsymbol{equal}{*}

\begin{icmlauthorlist}
\icmlauthor{Zhenmei Shi}{yyy}
\icmlauthor{Junyi Wei}{yyy}
\icmlauthor{Zhuoyan Xu}{yyy}
\icmlauthor{Yingyu Liang}{yyy,sch}
\end{icmlauthorlist}



\icmlaffiliation{yyy}{University of Wisconsin-Madison,}
\icmlaffiliation{sch}{The University of Hong Kong}

\icmlcorrespondingauthor{Zhenmei Shi, Yingyu Liang}{zhmeishi, yliang@cs.wisc.edu, yingyul@hku.hk}

\icmlkeywords{Machine Learning, ICML}

\vskip 0.3in
]



\printAffiliationsAndNotice{}  

\begin{abstract}
Large language models (LLM) have emerged as a powerful tool for AI, with the key ability of in-context learning (ICL), where they can perform well on unseen tasks based on a brief series of task examples without necessitating any adjustments to the model parameters. One recent interesting mysterious observation is that models of different scales may have different ICL behaviors: larger models tend to be more sensitive to noise in the test context. This work studies this observation theoretically aiming to improve the understanding of LLM and ICL. We analyze two stylized settings: (1) linear regression with one-layer single-head linear transformers and (2) parity classification with two-layer multiple attention heads transformers (non-linear data and non-linear model). In both settings, we give closed-form optimal solutions and find that smaller models emphasize important hidden features while larger ones cover more hidden features; thus, smaller models are more robust to noise while larger ones are more easily distracted, leading to different ICL behaviors. This sheds light on where transformers pay attention to and how that affects ICL. Preliminary experimental results on large base and chat models provide positive support for our analysis. 
\end{abstract}

\section{Introduction}
As large language models (LLM), e.g., ChatGPT~\citep{chatgpt} and GPT4~\citep{openai2023gpt4}, are transforming AI development with potentially profound impact on our societies, it is critical to understand their mechanism for safe and efficient deployment. An important emergent ability~\cite{wei2022emergent,an2023context}, which makes LLM successful, is \textit{in-context learning} (ICL), where models are given a few exemplars of input–label pairs as part of the prompt before evaluating some new input. More specifically, ICL is a few-shot~\cite{brown2020language} evaluation method without updating parameters in LLM. Surprisingly, people find that, through ICL, LLM can perform well on tasks that have never been seen before, even without any finetuning. It means LLM can adapt to wide-ranging downstream tasks under efficient sample and computation complexity. The mechanism of ICL is different from traditional machine learning, such as supervised learning and unsupervised learning. For example, in neural networks, learning usually occurs in gradient updates, whereas there is only a forward inference in ICL and no gradient updates. Several recent works, trying to answer why LLM can learn in-context, argue that LLM secretly performs or simulates gradient descent as meta-optimizers with just a forward pass during ICL empirically~\cite{dai2022can,von2023transformers,malladi2023fine} and theoretically~\cite{zhang2023trained,ahn2023transformers,mahankali2023one,cheng2023transformers,bai2023transformers,huang2023context,li2023transformers,guo2023transformers,wu2023many}. Although some insights have been obtained, the mechanism of ICL deserves further research to gain a better understanding.

Recently, there have been some important and surprising observations~\cite{min2022rethinking,pan2023what,wei2023larger,shi2023large} that cannot be fully explained by existing studies. In particular, \citet{shi2023large} finds that LLM is not robust during ICL and can be easily distracted by an irrelevant context. Furthermore, \citet{wei2023larger} shows that when we inject noise into the prompts, the larger language models may have a worse ICL ability than the small language models, and conjectures that the larger language models may overfit into the prompts and forget the prior knowledge from pretraining, while small models tend to follow the prior knowledge. On the other hand, \citet{min2022rethinking,pan2023what} demonstrate that injecting noise does not affect the in-context learning that much for smaller models, which have a more strong pretraining knowledge bias. 
To improve the understanding of the ICL mechanism, to shed light on the properties and inner workings of LLMs, and to inspire efficient and safe use of ICL, we are interested in the following question: 
\begin{center}
\textit{Why do larger language models do in-context learning differently?}   
\end{center}
To answer this question, we study two settings: (1) one-layer single-head linear self-attention network~\cite{schlag2021linear,von2023transformers,akyurek2023what,ahn2023transformers,zhang2023trained,mahankali2023one,wu2023many} pretrained on linear regression in-context tasks~\cite{garg2022can,raventos2023pretraining,von2023transformers,akyurek2023what,bai2023transformers,mahankali2023one,zhang2023trained,ahn2023transformers,pmlr-v202-li23l,huang2023context,wu2023many}, with rank constraint on the attention weight matrices for studying the effect of the model scale; (2) two-layer
multiple-head transformers~\cite{li2023transformers} pretrained on sparse parity classification in-context tasks, comparing small or large head numbers for studying the effect of the model scale. 
In both settings, we give the closed-form optimal solutions. We show that smaller models emphasize important hidden features while larger models cover more features, e.g., less important features or noisy features. 
Then, we show that smaller models are more robust to label noise and input noise during evaluation, while larger models may easily be distracted by such noises, so larger models may have a worse ICL ability than smaller ones. 

We also conduct in-context learning experiments on five prevalent NLP tasks utilizing various sizes of the Llama model families \cite{touvron2023llama,touvron2023llama2}, whose results are consistent with previous work~\cite{min2022rethinking,pan2023what,wei2023larger} and our analysis. 

\textbf{Our contributions and novelty over existing work:} 
\begin{itemize}
    \item We formalize new stylized theoretical settings for studying ICL and the scaling effect of LLM. See \Cref{sec:linear} for linear regression and \Cref{sec:parity} for parity. 
    \item We characterize the optimal solutions for both settings (\Cref{thm:low_opt} and \Cref{theorem:opt_parity}).
    \item The characterizations of the optimal elucidate different attention paid to different hidden features, which then leads to the different ICL behavior (\Cref{thm:mse}, \Cref{prop:diff}, \Cref{theorem:parity_decouple}). 
    \item We further provide empirical evidence on large base and chat models corroborating our theoretical analysis (\Cref{fig:nlp_in_context}, \Cref{fig:nlp_in_context_nonchat}). 
\end{itemize}

Note that previous ICL analysis paper may only focus on (1) the approximation power of transformers~\cite{garg2022can,panigrahi2023trainable,guo2023transformers,bai2023transformers,cheng2023transformers}, e.g., constructing a transformer by hands which can do ICL, or (2) considering one-layer single-head linear self-attention network learning ICL on linear regression~\cite{von2023transformers,akyurek2023what,mahankali2023one,zhang2023trained,ahn2023transformers,wu2023many}, and may not focus on the robustness analysis or explain the different behaviors. In this work, (1) we extend the linear model linear data analysis to the non-linear model and non-linear data setting, i.e.,
two-layer multiple-head transformers leaning ICL on sparse parity classification and (2) we have a rigorous behavior difference analysis under two settings, which explains the empirical observations and provides more insights into the effect of attention mechanism in ICL.  

\section{Related Work}
\textbf{Large language model.} Transformer-based~\cite{vaswani2017attention} neural networks have rapidly emerged as the primary machine learning architecture for tasks in natural language processing. 
Pretrained transformers with billions of parameters on broad and varied datasets are called large language models (LLM) or foundation models~\cite{bommasani2021opportunities}, e.g., BERT~\cite{devlin2018bert}, PaLM~\cite{chowdhery2022palm}, Llama\cite{touvron2023llama}, ChatGPT~\citep{chatgpt}, GPT4~\citep{openai2023gpt4} and so on. LLM has shown powerful general intelligence~\cite{bubeck2023sparks} in various downstream tasks. To better use the LLM for a specific downstream task, there are many adaptation methods, such as adaptor~\cite{hu2021lora,zhang2023llama,gao2023llama,shi2023the}, calibration~\cite{zhao2021calibrate,zhou2023lima}, multitask finetuning~\cite{gao2021making,xu2023improving,von2023transformers,zhuoyan}, prompt tuning~\cite{gao2020making,lester2021power}, instruction tuning~\cite{li2021prefix,chung2022scaling,mishra2022cross}, symbol tuning~\cite{wei2023symbol}, black-box tuning~\cite{sun2022black}, chain-of-thoughts~\cite{wei2022chain,khattab2022demonstrate,yao2023tree,zheng2023take}, scratchpad~\cite{nye2021show}, reinforcement learning from human feedback (RLHF)~\cite{ouyang2022training} and many so on. 

\textbf{In-context learning.} One important emergent ability~\cite{wei2022emergent} from LLM is in-context learning (ICL)~\cite{brown2020language}. Specifically, when presented with a brief series of input-output pairings (known as a prompt) related to a certain task, they can generate predictions for test scenarios without necessitating any adjustments to the model's parameters. ICL is widely used in broad scenarios, e.g., reasoning~\cite{zhou2022teaching}, negotiation~\cite{fu2023improving}, self-correction~\cite{pourreza2023din}, machine translation~\cite{agrawal2022context} and so on. Many works trying to improve the ICL and zero-shot ability of LLM~\cite{min2021metaicl,wang2022super,wei2022finetuned,iyer2022opt}.
There is a line of insightful works to study the mechanism of transformer learning~\cite{geva2021transformer,xie2022an,garg2022can,jelassi2022vision,arora2023theory,li2023a,pmlr-v202-li23p,allen2023physics,luo2023understanding,tian2023scan,tian2023joma,zhou2023algorithms,bietti2023birth,xsl24,gll+24b,gll+24a,gll+24c,gls+24a,gls+24b} and in-context learning~\cite{dai2022can,mahankali2023one,raventos2023pretraining,bai2023transformers,ahn2023transformers,von2023transformers,pan2023what,li2023transformers,pmlr-v202-li23l,li2023dissecting,akyurek2023what,zhang2023study,zhang2023trained,huang2023context,cheng2023transformers,wibisono2023role,wu2023many,guo2023transformers,reddy2023mechanistic} empirically and theoretically. On the basis of these works, our analysis takes a step forward to show the ICL behavior difference under different scales of language models. 

\section{Preliminary}
\textbf{Notations.} 
We denote $[n]:= \{1,2,\dots, n\}$. For a positive semidefinite matrix $\A$, we denote
$\|\x\|_\A^2 := \x^\top \A \x$ as the norm induced by a positive definite matrix $\A$. We denote $\|\cdot\|_F$ as the Frobenius norm. $\diag()$ function will map a vector to a diagonal matrix or map a matrix to a vector with its diagonal terms. 

\textbf{In-context learning.} We follow the setup and notation of the problem in~\citet{zhang2023trained,mahankali2023one,ahn2023transformers,huang2023context,wu2023many}. 
In the pretraining stage of ICL, the model is pretrained on prompts. A prompt from a task $\tau$ is formed by $N$ examples $(\x_{\tau, 1}, y_{\tau, 1}), \dots, (\x_{\tau, N}, y_{\tau, N})$ and a query token $\x_{\tau, q}$ for prediction, where for any $i \in [N] $ we have $ y_{\tau, i} \in \R$ and $\x_{\tau, i}, \x_{\tau, q} \in \R^d$. The embedding matrix $\bE_\tau$, the label vector $\y_\tau$, and the input matrix $\X_\tau$ are defined as: 
\begin{align*}
    \bE_\tau  
        := &  \begin{pmatrix}
        \x_{\tau, 1} & \x_{\tau, 2} & \dots & \x_{\tau, N} & \x_{\tau, q} \\
        y_{\tau, 1} & y_{\tau, 2} & \dots & y_{\tau, N} & 0 \\
        \end{pmatrix}  
        \in \R^{(d+1)\times (N+1)}, \\
    \y_\tau := & [y_{\tau, 1}, \dots, y_{\tau, N}]^\top \in \R^{N},  
    ~\quad\quad\quad y_{\tau, q} \in \R, \\
    \X_\tau := & [\x_{\tau, 1}, \dots,  \x_{\tau, N}]^\top \in \R^{N \times d}, 
    \quad\quad \x_{\tau, q} \in \R^d.
\end{align*}
Given prompts represented as $\bE_\tau$'s and the corresponding true labels $y_{\tau, q}$'s, the pretraining aims to find a model whose output on $\bE_\tau$ matches $y_{\tau, q}$. After pretraining, the evaluation stage applies the model to a new test prompt (potentially from a different task) and compares the model output to the true label on the query token. 

Note that our pretraining stage is also called learning to learn in-context~\cite{min2021metaicl} or in-context training warmup~\cite{dong2022survey} in existing work. 
Learning to learn in-context is the first step to understanding the mechanism of ICL in LLM following previous works~\cite{raventos2023pretraining,zhou2023algorithms,zhang2023trained,mahankali2023one,ahn2023transformers,huang2023context,li2023transformers,wu2023many}.  

\textbf{Linear self-attention networks.}
The linear self-attention network has been widely studied~\cite{schlag2021linear,von2023transformers,akyurek2023what,ahn2023transformers,zhang2023trained,mahankali2023one,wu2023many,ahn2023linear}, and will be used as the learning model or a component of the model in our two theoretical settings. It is defined as:
\begin{align}\label{eq:single_attention}
    f_{\textup{LSA},\theta}(\bE) = \left[\bE + \W^{PV} \bE \cdot {\bE^\top \W^{KQ} \bE \over \rho}\right], 
\end{align}
where $\theta = (\W^{PV}, \W^{KQ})$, $\bE \in \R^{(d+1)\times (N+1)}$ is the embedding matrix of the input prompt, and $\rho$ is a normalization factor 
set to be the length of examples, i.e., $\rho = N$ during pretraining.  
Similar to existing work, for simplicity, 
we have merged the projection and value matrices into $\W^{PV}$, and merged the key and query matrices into $\W^{KQ}$, and have a residual connection in our LSA network. The prediction of the network for the query token $\x_{\tau,q}$ will be the bottom right entry of the matrix output, i.e., the entry at location ${(d+1),(N+1)}$, 
while other entries are not relevant to our study and thus are ignored. So only part of the model parameters are relevant. To see this, let us denote
\begin{align*}
& \W^{PV} = \begin{pmatrix}
\W^{PV}_{11} & \w^{PV}_{12} \\
(\w^{PV}_{21})^\top & w^{PV}_{22} 
\end{pmatrix}\in \R^{(d+1) \times (d+1)}, \\ 
& \W^{KQ} = \begin{pmatrix}
\W^{KQ}_{11} & \w^{KQ}_{12} \\
(\w^{KQ}_{21})^\top & w^{KQ}_{22} 
\end{pmatrix} \in \R^{(d+1) \times (d+1)}, 
\end{align*}
where $\W^{PV}_{11}, \W^{KQ}_{11} \in \R^{d\times d}$; $\w^{PV}_{12}, \w^{PV}_{21}, \w^{KQ}_{12}, \w^{KQ}_{21}  \in \R^{d}$; and $w^{PV}_{22}, w^{KQ}_{22} \in \R$. Then the prediction is:
\begin{align} \label{eq:query-prediction}
    \widehat{y}_{\tau, q} = & f_{\textup{LSA},\theta}(\bE)_{(d+1),(N+1)} \\
    = & \begin{pmatrix}
(\w^{PV}_{21})^\top & w^{PV}_{22} \end{pmatrix} \left(\bE \bE^\top \over \rho\right) \begin{pmatrix}
\W^{KQ}_{11}  \\
(w^{KQ}_{21})^\top 
\end{pmatrix} \x_{\tau, q}. \nonumber
\end{align}

\section{Linear Regression}\label{sec:linear}
In this section, we consider the linear regression task for in-context learning which is widely studied empirically~\cite{garg2022can,raventos2023pretraining,von2023transformers,akyurek2023what,bai2023transformers} and theoretically~\cite{mahankali2023one,zhang2023trained,ahn2023transformers,pmlr-v202-li23l,huang2023context,wu2023many}. 

\textbf{Data and task.}
For each task $\tau$, we assume for any $i \in [N]$ tokens $\x_{\tau, i}, \x_{\tau, q} \iid \cN(0, \Lambda)$, where $\Lambda$ is the covariance matrix. We also assume a $d$-dimension task weight $\w_\tau \iid \cN(0, I_{d\times d})$ and the labels are given by $ y_{\tau, i} = \inner{\w_\tau, \x_{\tau, i}}$ and $y_{\tau, q} = \inner{\w_\tau, \x_{\tau, q}}$. 

\textbf{Model and loss.}
We study a one-layer single-head linear self-attention transformer (LSA) defined in \Cref{eq:single_attention} and we use $\widehat{y}_{\tau,q} := f_{\textup{LSA},\theta}(\bE)_{(d+1),(N+1)} $ as the prediction. We consider the mean square error (MSE) loss so that the empirical risk over $B$ independent prompts is defined as 
$$
    \widehat{\mathcal{L}}(f_\theta) := {1
    \over 2B} \sum_{\tau = 1}^B \left(\widehat{y}_{\tau,q} - \inner{\w_\tau, \x_{\tau,q}} \right)^2. 
$$

\textbf{Measure model scale by rank.}
We first introduce a lemma from previous work that simplifies the MSE and justifies our measurement of the model scale. 
For notation simplicity, we denote $\U = \W^{KQ}_{11}, u = w^{PV}_{22}$.

\begin{lemma}[Lemma A.1 in~\citet{zhang2023trained}]\label{lem:loss_simple} 
Let $\Gamma := \left(1+{1\over N}\right)\Lambda + {1\over N} \tr(\Lambda) I_{d \times d} \in \R^{d\times d}$. Let
\begin{align*}
   \mathcal{L}(f_{\textup{LSA},\theta}) = & \lim_{B \rightarrow \infty} \widehat{\mathcal{L}}(f_{\textup{LSA},\theta}) \\
   = & {1\over 2} \E_{\w_\tau, \x_{\tau,1},\dots,\x_{\tau,N},\x_{\tau,q}}\left[\left(\widehat{y}_{\tau,q} - \inner{\w_\tau, \x_{\tau,q}} \right)^2\right], \nonumber \\
   \tilde{\ell}(\U, u) = & \tr \left[{1\over 2} u^2 \Gamma\Lambda \U \Lambda \U^\top - u \Lambda^2 \U^\top\right],
\end{align*}
we have $\mathcal{L}(f_{\textup{LSA},\theta}) = \tilde{\ell}(\U, u) + C$, where $C$ is a constant independent with $\theta$.  
\end{lemma}

\Cref{lem:loss_simple} tells us that the loss only depends on $u\U$. If we consider non-zero $u$, w.l.o.g, letting $u = 1$, then we can see that the loss only depends on $\U \in \R^{d \times d}$, 
$$
    \mathcal{L}(f_{\textup{LSA},\theta}) = \tr \left[{1\over 2}\Gamma\Lambda \U \Lambda \U^\top - \Lambda^2 \U^\top\right].
$$
Note that $\U=\W^{KQ}_{11}$, then it is natural to measure the size of the model by rank of $\U$. 
Recall that we merge the key matrix and the query matrix in attention together, i.e., $\W^{KQ} = (\W^K)^\top \W^Q$. Thus, a low-rank $\U$ is equivalent to the constraint $\W^K, \W^Q \in \R^{r \times d}$ where $r \ll d$. The low-rank key and query matrix are practical and have been widely studied~\cite{hu2021lora,chen2021scatterbrain,bhojanapalli2020low,fan2021lighter,dass2023vitality,shi2023domain}. Therefore, we use $r = \rank(\U)$ to measure the scale of the model, i.e., larger $r$ representing larger models. To study the behavior difference under different model scale, we will analyze $\U$ under different rank constraints.

\subsection{Low Rank Optimal Solution}\label{sec:low-opt}
Since the token covariance matrix $\Lambda$ is positive semidefinite symmetric, we have eigendecomposition $\Lambda = \Q \eigen \Q^\top$, where $\Q$ is an orthonormal matrix containing eigenvectors of $\Lambda$ and $\eigen$ is a sorted diagonal matrix with non-negative entries containing eigenvalues of $\Lambda$, denoting as $\eigen = \diag([\lambda_1, \dots, \lambda_d])$, where $\lambda_{1} \ge \dots \ge \lambda_{d} \ge 0$.
Then, we have the following theorem. 
\todo[inline,color=gray!10]{
{\normalsize
\vspace{-0.5em}
\begin{restatable}[Optimal rank-$r$ solution for regression]{thm}{lowopt}\label{thm:low_opt}
Recall the loss function $\tilde{\ell}$ in \Cref{lem:loss_simple}. Let 
$$
    \U^*, u^* = \argmin_{\U \in \R^{d\times d}, \rank(\U) \le r, u \in \R} \tilde{\ell}(\U, u).
$$
Then $\U^* = c \Q \V^* \Q^\top, u = {1\over c}$, where $c$ is any nonzero constant, and $\V^* = \diag([v^*_1, \dots, v^*_d])$ satisfies for any $i \le r, v^*_i = {N \over {\left(N+1 \right)\lambda_i + {\tr(\eigen)}}}$ and for any $i > r, v^*_i = 0$.  
\end{restatable}
}
}

\begin{proof}[Proof sketch of \Cref{thm:low_opt}]
We defer the full proof to \Cref{app:thm:low_opt}.
The proof idea is that we can decompose the loss function into different ranks, so we can keep the direction by their sorted  ``variance'', i.e., 
\begin{align*}
    \argmin_{\U \in \R^{d\times d}, \rank(\U) \le r, u \in \R} \tilde{\ell}(\U, u) =& \sum_{i=1}^d  T_i\lambda_i^2\left(v^*_i - {1\over T_i}\right)^2,
\end{align*}
where $T_i = \left(1+ {1\over N} \right)\lambda_i + {\tr(\eigen)\over N}$.
We have that $v^*_i \ge 0$ for any $i \in [d]$ and if $v^*_i > 0$, we have $v^*_i = {1\over T_i}$. Denote $g(x) = x^2\left({1\over {\left(1+ {1\over N} \right)x + {\tr(\eigen)\over N}}}\right)$. We get the conclusion by $g(x)$ is an increasing function on $[0,\infty)$. 
\end{proof}

\Cref{thm:low_opt} gives the closed-form optimal rank-$r$ solution of one-layer single-head linear self-attention transformer learning linear regression ICL tasks. 
Let $f_{\textup{LSA},\theta}$ denote the optimal rank-$r$ solution corresponding to the $\U^*, u^*$ above. 
In detail, 
the optimal rank-$r$ solution $f_{\textup{LSA},\theta}$ satisfies 
\begin{align} \label{eq:optimal-rank-r}
& \W^{*PV} = \begin{pmatrix}
0_{d\times d} & 0_d \\
0_d^\top & u 
\end{pmatrix},  \W^{*KQ} = \begin{pmatrix}
\U^* & 0_d \\
0_d^\top & 0 
\end{pmatrix}.
\end{align}

\textbf{What hidden features does the model pay attention to?}
\Cref{thm:low_opt} shows that the optimal rank-$r$ solution indeed is the truncated version of the optimal full-rank solution, keeping only the most important feature directions (i.e., the first $r$ eigenvectors of the token covariance matrix). In detail, (1) for the optimal full-rank solution, we have for any $i \in [d], v^*_i = {N \over {\left(N+1 \right)\lambda_i + {\tr(\eigen)}}}$; (2) for the optimal rank-$r$ solution, we have for any $i \le r, v^*_i = {N \over {\left(N+1 \right)\lambda_i + {\tr(\eigen)}}}$ and for any $i > r, v^*_i = 0$. 
That is, the small rank-$r$ model keeps only the first $r$ eigenvectors (viewed as hidden feature directions) and does not cover the others, while larger ranks cover more hidden features, and the large full rank model covers all features.

Recall that the prediction depends on $\U^* \x_{\tau,q} = c \Q \V^* \Q^\top \x_{\tau,q}$; see \Cref{eq:query-prediction} and (\ref{eq:optimal-rank-r}). So the optimal rank-$r$ model only uses the components on the first $r$ eigenvector directions to do the prediction in evaluations. When there is noise distributed in all directions, a smaller model can ignore noise and signals along less important directions but still keep the most important directions. Then it can be less sensitive to the noise, as empirically observed. This insight is formalized in the next subsection.

\subsection{Behavior Difference}\label{sec:diff}
We now formalize our insight into the behavior difference based on our analysis on the optimal solutions. 
We consider the evaluation prompt to have $M$ examples (may not be equal to the number of examples $N$ during pretraining for a general evaluation setting), and assume noise in labels to facilitate the study of the behavior difference (our results can be applied to the noiseless case by considering noise level $\sigma = 0$). Formally, the evaluation prompt is:
\begin{align*}
    \widehat{\bE}
        := &  \begin{pmatrix}
        \x_{1} & \x_{2} & \dots & \x_{M} & \x_{q} \\
        y_{1} & y_{2} & \dots & y_{M} & 0 \\
        \end{pmatrix} \in \R^{(d+1)\times (M+1)} \\
        = &  \begin{pmatrix}
        \x_{1} &\dots & \x_{M} & \x_{q} \\
        \inner{\w, \x_{1}} + \epsilon_1 &  \dots & \inner{\w, \x_{M}} + \epsilon_M & 0 \\
        \end{pmatrix} ,
\end{align*}
where $\w$ is the weight for the evaluation task, and for any $i \in [M]$, the label noise $\epsilon_i \iid \cN(0, \sigma^2)$.

Recall $\Q$ are eigenvectors of $\Lambda$, i.e., $\Lambda = \Q \eigen \Q^\top$ and $\eigen = \diag([\lambda_1, \dots, \lambda_d])$.
In practice, we can view the large variance part of $\x$ (top $r$ directions in $\Q$) as a useful signal (like words ``positive'', ``negative''), and the small variance part (bottom $d-r$ directions in $\Q$) as the less important or useless information (like words ``even'', ``just''). 

Based on such intuition, we can decompose the evaluation task weight $\w$ accordingly: $\w =  \Q(\s + \xi)$, where the $r$-dim truncated vector $\s \in \R^d$ has $\s_i = 0$ for any $r < i \le d$, and the residual vector $\xi \in \R^d$ has $\xi_i = 0$ for any $1 \le i \le r$. The following theorem (proved in \Cref{app:thm:mse}) quantifies the evaluation loss at different model scales $r$ which can explain the scale's effect. 
\todo[inline,color=gray!10]{
{\normalsize
\vspace{-0.5em}
\begin{restatable}[Behavior difference for regression]{thm}{mse}\label{thm:mse}
Let $\w = \Q(\s + \xi) \in \R^d$ where $\s, \xi \in \R^d$ are truncated and residual vectors defined above. The optimal rank-$r$ solution $f_{\textup{LSA},\theta}$ in \Cref{thm:low_opt} satisfies:
\begin{align*}
    &  \mathcal{L}(f_{\textup{LSA},\theta};\widehat{\bE}) \\ := & \E_{\x_1,\epsilon_1,\dots,\x_M,\epsilon_M,\x_q}\left( f_{\textup{LSA},\theta}(\widehat{\bE}) - \inner{\w,\x_q}\right)^2 \\
    = & {1\over M} \|\s \|_{(\V^*)^2 \eigen^3}^2 + {1\over M} \left(\|\s + \xi\|_\eigen^2 + \sigma^2\right)   \tr\left(    (\V^*)^2 \eigen^2\right) \\
    & + \|\xi\|_\eigen^2  + \sum_{i \in[r]} \s_i^2 \lambda_i\left( \lambda_i v^*_i - 1 \right)^2. 
\end{align*}
\end{restatable}
}
}

\textbf{Implications.} If $N$ is large enough with $N \lambda_r \gg \tr(\eigen)$ (which is practical as we usually pretrain networks on long text), then
\begin{align*}
    \mathcal{L}(f_{\textup{LSA},\theta};\widehat{\bE}) {\approx} & \|\xi\|_\eigen^2  + {1\over M} \left((r+1)\|\s \|_{\eigen}^2 +  r\|\xi\|_\eigen^2 + r\sigma^2   \right).
\end{align*}

The first term $\|\xi\|_\eigen^2$ is due to the residual features not covered by the network, so it decreases for larger $r$ and becomes $0$ for full-rank $r=d$.  
The second term ${1\over M}(\cdot)$ 
is significant since we typically have limited examples in evaluation, e.g., $M = 16 \ll N$. Within it, $(r+1)\|\s \|_{\eigen}^2$ corresponds to the first $r$ directions, and $r \sigma^2$ corresponds to the label noise. These increase for larger $r$. 
So there is a trade-off between the two error terms when scaling up the model: for larger $r$ the first term decreases while the second term increases. This depends on whether more signals are covered or more noise is kept when increasing the rank $r$. 

To further illustrate the insights, we consider the special case when the model already covers all useful signals in the evaluation task: $\w =\Q\s$, i.e., the label only depends on the top $r$ features (like ``positive'', ``negative'' tokens). Our above analysis implies that a larger model will cover more useless features and keep more noise, and thus will have worse performance. This is formalized in the following theorem (proved in \Cref{app:thm:mse}).

\todo[inline,color=gray!10]{
{\normalsize
\vspace{-0.5em}
\begin{restatable}[Behavior difference for regression, special case]{thm}{diff}\label{prop:diff}
Let $0\le {r} \le {r'} \le d$ and $\w = \Q\s$ where $\s$ is ${r}$-dim truncated vector. Denote the optimal rank-${r}$ solution as $f_1$ and the optimal rank-${r'}$ solution as $f_2$. Then, 
\begin{align*}
    & \mathcal{L}(f_2;\widehat{\bE}) - \mathcal{L}(f_1;\widehat{\bE}) \\
    = & {1\over M} \left(\|\s \|_\eigen^2 + \sigma^2\right)   \left(\sum_{i = {r}+1}^{{r'}} \left({N \lambda_i \over {\left(N+1 \right)\lambda_i + {\tr(\eigen)}}}\right)^2\right).
\end{align*}
\end{restatable}
}
}

\textbf{Implications.} By \Cref{prop:diff}, in this case,   
\begin{align*}
    \mathcal{L}(f_2;\widehat{\bE}) - \mathcal{L}(f_1;\widehat{\bE}) \approx  \underbrace{{{r'} - {r}\over M}\|\s \|_\eigen^2}_{\text{input noise}} + \underbrace{{{r'} - {r}\over M}\sigma^2}_{\text{label noise}} \label{eq:diff}.
\end{align*}
We can decompose the above equation to input noise and label noise, and we know that $\|\s \|_\eigen^2 + \sigma^2$ only depends on the intrinsic property of evaluation data and is independent of the model size. 
When we have a larger model (larger ${r'}$), we will have a larger evaluation loss gap between the large and small models. It means larger language models may be easily affected by the label noise and input noise and may have worse in-context learning ability, while smaller models may be more robust to these noises as they only emphasize important signals. Moreover, if we increase the label noise scale $\sigma^2$ on purpose, the larger models will be more sensitive to the injected label noise. This is consistent with the observation in \citet{wei2023larger,shi2023large} and our experimental results in \Cref{sec:exp}.

\section{Sparse Parity Classification}\label{sec:parity}

We further consider a more sophisticated setting with nonlinear data which necessitates nonlinear models. Viewing sentences as generated from various kinds of thoughts and knowledge that can be represented as vectors in some hidden feature space, we consider the classic data model of dictionary learning or sparse coding, which has been widely used for text and images~\cite{Olshausen1997SparseCW,vinje2000sparse,blei2003latent}.  
Furthermore, beyond linear separability, we assume the labels are given by the $(d, 2)$-sparse parity on the hidden feature vector, which is the high-dimensional generalization of the classic XOR problem. Parities are a canonical family of highly non-linear learning problems and recently have been used in many recent studies on neural network learning~\cite{dm20,beg+22,swl22,shi2023provable}.

\textbf{Data and task.}
Let $\cX = \R^d$ be the input space, and $\cY = \{\pm 1\}$ be the label space. Suppose $\G \in \R^{d \times d}$ is an unknown dictionary with $d$ columns that can be regarded as features; for simplicity, assume $\G$ is
orthonormal. Let $\phi \in \{\pm 1\}^d$
be a hidden vector that indicates the presence of each feature. 
The data are generated as follows: for each task $\tau$, generate two task indices $\t_\tau = (i_\tau,j_\tau)$ which determines a distribution $\cT_\tau$; then for this task, draw examples by $\phi \sim \cT_\tau$, and setting $\x = \G \phi$ (i.e., dictionary learning data), $y = \phi_{i_\tau} \phi_{j_\tau}$ (i.e., XOR labels). 

We now specify how to generate $\t_\tau$ and $\phi$. 
As some of the hidden features are more important than others, we let $A = [k]$ denote a subset of size $k$ corresponding to the important features. 
We denote the important task set as $S_1: = A\times A \setminus \{(l,l): l \in A\}$ and less important task set as $S_2: = [d]\times [d] \setminus (\{(l,l): l \in [d]\} \cup S_1)$. Then $\t_\tau$ is drawn uniformly from $S_1$ with probability $1-p_\cT$, and uniformly from $S_2$ with probability $p_\cT$, where $p_\cT \in [0, {1\over 2})$ is the less-important task rate. 
For the distribution of $\phi$, we assume $\phi_{[d]\setminus\{i_\tau,j_\tau\}}$ is drawn uniformly from $\{\pm 1\}^{d-2}$, and assume $\phi_{\{i_\tau,j_\tau\}}$ has good correlation (measured by a parameter $\gamma \in (0, {1\over 4})$) with the label to facilitate learning.  
Independently, we have 
\begin{align*}
    &\Pr[(\phi_{i_\tau},\phi_{j_\tau}) = (1,1)] = {1/ 4} + \gamma, \\
    &\Pr[(\phi_{i_\tau},\phi_{j_\tau}) = (1,-1)] = {1/ 4}, \\
    &\Pr[(\phi_{i_\tau},\phi_{j_\tau}) = (-1,1)] = {1/ 4}, \\
    &\Pr[(\phi_{i_\tau},\phi_{j_\tau}) = (-1,-1)] = {1/ 4} - \gamma.
\end{align*} 
Note that without correlation ($\gamma=0$), it is well-known sparse parities will be hard to learn, so we consider $\gamma >0$.

\textbf{Model.} 
Following~\citet{wu2023many}, we consider the reduced linear self-attention 
$
    f_{\textup{LSA},\theta}(\X, \y, \x_q) = {\y^\top \X \over N} \W^{KQ} \x_q 
$
(which is a reduced version of \Cref{eq:single_attention}), and also denote $\W^{KQ}$ as $\W$ for simplicity. 
It is used as the neuron in our two-layer multiple-head transformers:
\begin{align*}
    g(\X, \y, \x_q) = \sum_{i \in [m]} \a_i \act\left[ {\y^\top \X \over N} \W^{(i)} \x_q \right], 
\end{align*}
where $\act$ is ReLU activation, $\a = [\a_1, \dots, \a_{m}]^\top \in [-1,1]^{m}$, $\W^{(i)} \in \R^{d \times d}$ and $m$ is the number of attention heads. Denote its parameters as $\theta = (\a, \W^{(1)}, \dots, \W^{(m)})$. 

This model is more complicated as it uses non-linear activation, and also has two layers with multiple heads.

\textbf{Measure model scale by head number.} We use the attention head number $m$ to measure the model scale, as a larger $m$ means the transformer can learn more attention patterns. 
We consider hinge loss $\ell(z) = \max(0, 1 - z) $, 
and the population loss with weight-decay regularization:
\begin{align*}
    \cL^\lambda(g) = & \E\left[\ell\left( y_{q}\cdot g(\X, \y, \x_{q})\right) \right] + \lambda\left(\sum_{i \in [m]} \|\W^{(i)}\|_F^2 \right) \nonumber.
\end{align*}
Suppose $N \rightarrow \infty$ and let the optimal solution of $\cL^\lambda(g)$ be 
\begin{align*}
g^*= \argmin_{g} \quad \lim_{\lambda \rightarrow 0^+} \cL^\lambda(g).
\end{align*}

\subsection{Optimal Solution}
We first introduce some notations to describe the optimal. Let $\bin(\cdot)$ be the integer to binary function, e.g., $\bin(6) = 110$. Let  $\digit(z,i)$ denote the digit at the $i$-th position (from right to left) of $z$, e.g., $\digit(01000,4)=1$. 
We are now ready to characterize the optimal solution (proved in~\Cref{app:thm:opt_parity}).
\todo[inline,color=gray!10]{
{\normalsize
\vspace{-0.5em}
\begin{restatable}[Optimal solution for parity]{thm}{optparity}\label{theorem:opt_parity}
Consider $k = 2^{\nu_1}, d=2^{\nu_2}$, and let $g^*_1$ and $g^*_2$ denote the optimal solutions for $m = 2(\nu_1+1)$ and $m = 2(\nu_2+1)$, respectively.  

When $0< p_\cT < { {1\over 4} - \gamma \over {d(d-1) \over 2} ({1\over 4} + \gamma) +  {1\over 4} - \gamma}$,
$g^*_1$ neurons are a subset of $g^*_2$ neurons. Specifically, for any $i \in [2(\nu_2 + 1)]$, let $\V^{*,(i)}$ be diagonal matrix and 
\begin{itemize}
    \item For any $i\in [\nu_2]$ and $i_\tau \in [d]$, let $ \a^*_i = -1$ and $\V_{i_\tau,i_\tau}^{*,(i)} = (2\digit(\bin(i_\tau-1), i)-1)/(4\gamma)$.
    \item For $i = \nu_2 +1$ and any $i_\tau \in [d]$, let $ \a^*_i = +1$ and $\V_{i_\tau,i_\tau}^{*,(i)} = -\nu_j/(4\gamma)$ for $g^*_j$. 
    \item For $i \in [2(\nu_2 + 1)]\setminus [\nu_2 + 1]$, let $ \a^*_i= \a^*_{i-\nu_2- 1} $ and $\V^{*,(i)} = - \V^{*,(i-\nu_2- 1)}$. 
\end{itemize}
Let $\W^{*,(i)} = \G\V^{*,(i)}\G^\top$. Up to permutations, $g^*_2$ has neurons $(\a^*, \W^{*,(1)}, \dots, \W^{*,(m)})$ and $g^*_1$ has the $\{1,\dots, \nu_1, \nu_2+1, \nu_2+2 \dots, \nu_2+\nu_1+1, 2\nu_2+2 \}$-th neurons of $g^*_2$. 
\end{restatable}
}
}

\begin{proof}[Proof sketch of \Cref{theorem:opt_parity}]
The proof is challenging as the non-linear model and non-linear data. We defer the full proof to \Cref{app:thm:opt_parity}. The high-level intuition is transferring the optimal solution to patterns covering problems. For small $p_\cT$, the model will ``prefer'' to cover all patterns in $S_1$ first. When the model becomes larger, by checking the sufficient and necessary conditions, it will continually learn to cover non-important features. Thus, the smaller model will mainly focus on important features, while the larger model will focus on all features. 
\end{proof}

\textbf{Example for \Cref{theorem:opt_parity}.} When $\nu_2 = 3$, the optimal has $\a_1 = \a_2 = \a_3 = -1$, $\a_4 = +1$ and, 
$$
\V^{(1)} = {1/ 4\gamma} \cdot \diag([-1,+1,-1,+1,-1,+1,-1,+1]) 
$$
$$
\V^{(2)} = {1/ 4\gamma}  \cdot \diag([-1,-1,+1,+1,-1,-1,+1,+1])
$$
$$
\V^{(3)} = {1/ 4\gamma}  \cdot \diag([-1,-1,-1,-1,+1,+1,+1,+1])
$$
$$
\V^{(4)} = {3 / 4\gamma}  \cdot \diag([-1,-1,-1,-1,-1,-1,-1,-1])
$$
and $\V^{(i+4)} = - \V^{(i)},  \a_{i+4}=\a_i $ for $i \in [4]$. 

On the other hand, the optimal $g^*_1$ for $\nu_1 = 1$ has the $\{1, 4, 5, 8\}$-th neurons of $g^*_2$. 

By carefully checking, we can see that the neurons in $g^*_1$ (i.e., the $\{1, 4, 5, 8\}$-th neurons of $g^*_2$) are used for parity classification task from $S_1$, i.e, label determined by the first $k=2^{\nu_1} = 2$ dimensions. With the other neurons (i.e., the $\{2, 3, 6, 7\}$-th neurons of $g^*_2$), $g^*_2$ can further do parity classification on the task from $S_2$, label determined by any two dimensions other than the first two dimensions.

\textbf{What hidden features does the model pay attention to?}
\Cref{theorem:opt_parity} gives the closed-form optimal solution of two-layer multiple-head transformers learning sparse-parity ICL tasks. 
It shows the optimal solution of the smaller model indeed is a sub-model of the larger optimal model. In detail, the smaller model will mainly learn all important features, while the larger model will learn more features.  
This again shows a trade-off when increasing the model scale: larger models can learn more hidden features which can be beneficial if these features are relevant to the label, but also potentially keep more noise which is harmful.

\begin{figure*}[!ht]
\begin{center}
\includegraphics[width=0.91\linewidth]{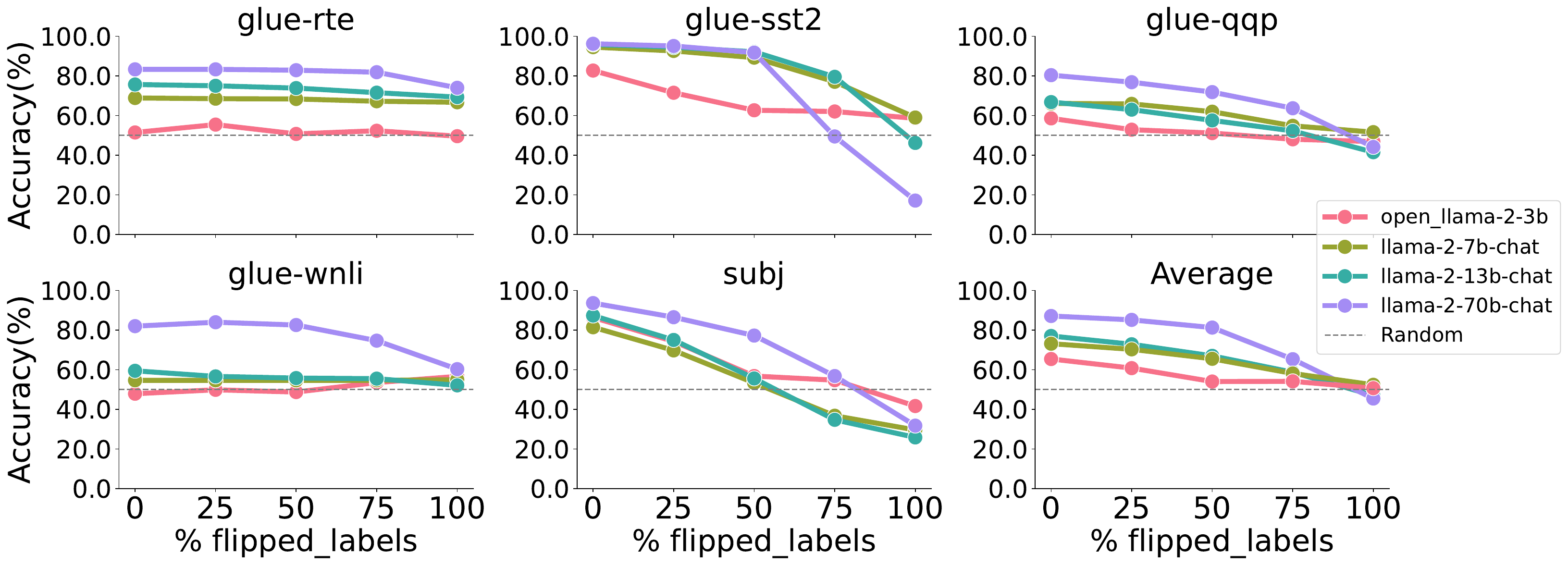}
\end{center}
\caption{  Larger models are easier to be affected by noise (flipped labels) and override pretrained biases than smaller models for different datasets and model families (chat/with instruct turning). Accuracy is calculated over 1000 evaluation prompts per dataset and over 5 runs with different random seeds for each evaluation, using $M = 16$ in-context exemplars. }
\label{fig:nlp_in_context}
\end{figure*}

\begin{figure*}[!ht]
\begin{center}
\includegraphics[width=0.91\linewidth]{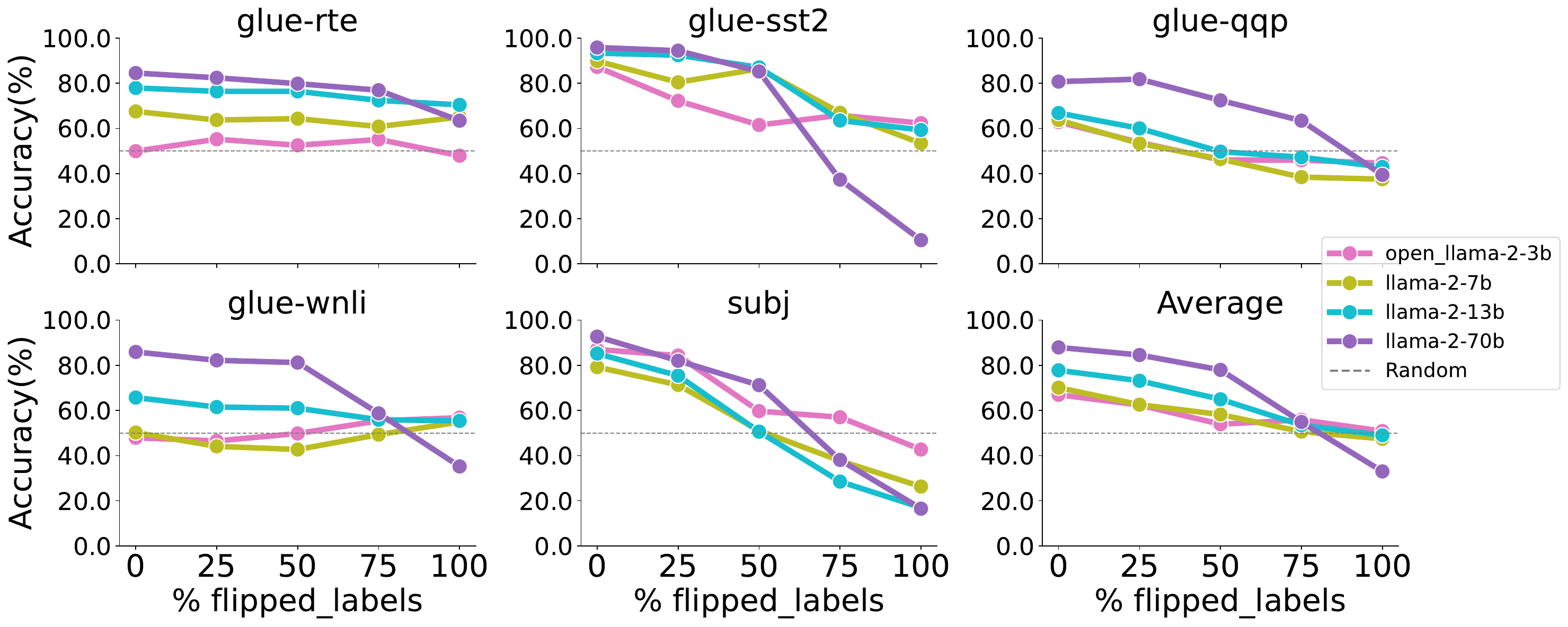}
\end{center}
\caption{  Larger models are easier to be affected by noise (flipped labels) and override pretrained biases than smaller models for different datasets and model families (original/without instruct turning). Accuracy is calculated over 1000 evaluation prompts per dataset and over 5 runs with different random seeds for each evaluation, using $M = 16$ in-context exemplars. }
\label{fig:nlp_in_context_nonchat}
\end{figure*}

\begin{figure}[!ht]
\begin{center}
\includegraphics[width=\linewidth]{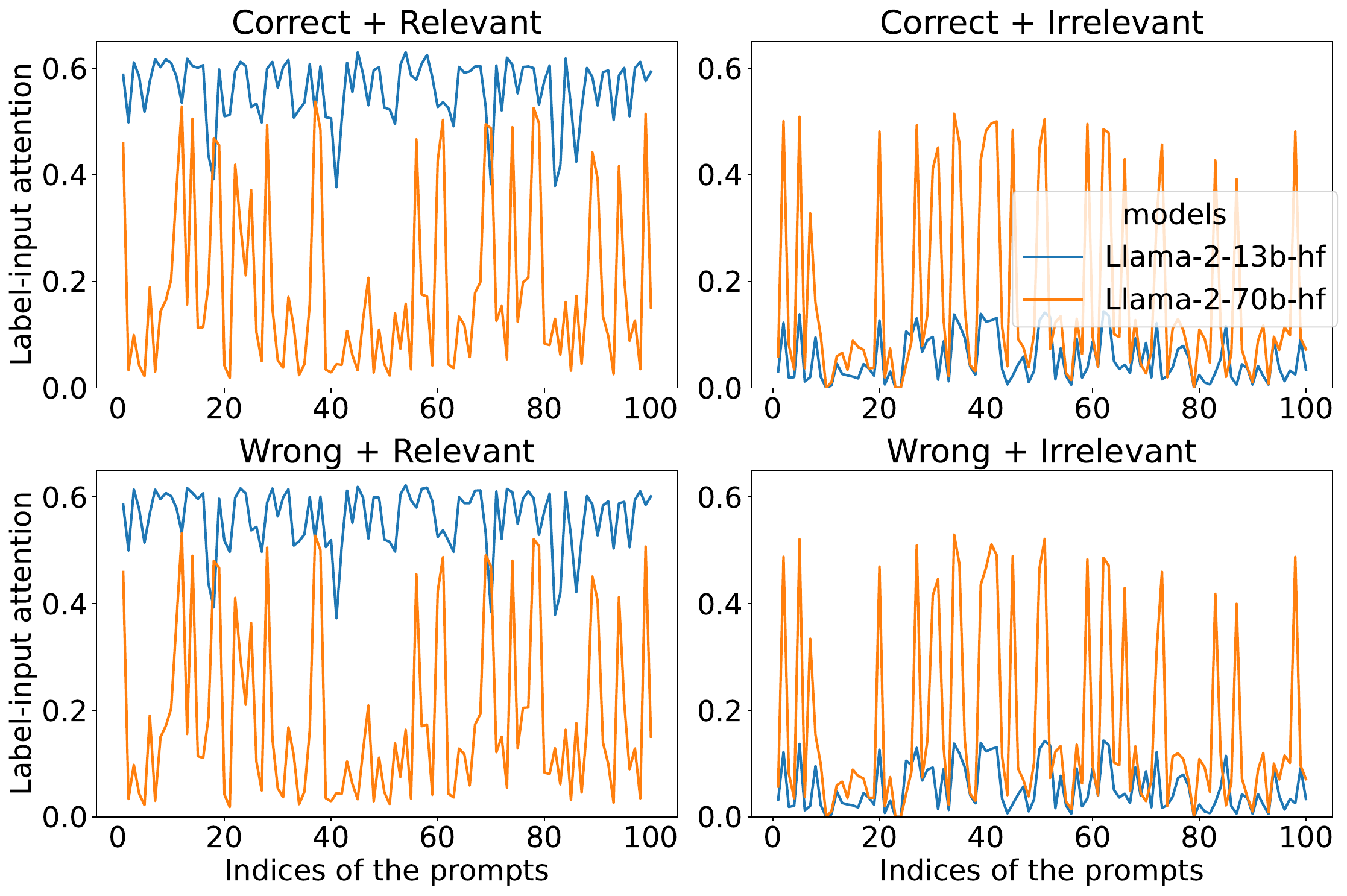}
\end{center}
\caption{
The magnitude of attention between the labels and input sentences in Llama 2-13b and 70b on 100 evaluation prompts; see the main text for the details. $x$-axis: indices of the prompts. $y$-axis: the norm of the last row of attention maps in the final layer.
Correct: original label; wrong: flipped label; relevant: original input sentence; irrelevant: irrelevant sentence from other datasets. 
The results show that larger models focus on both sentences, while smaller models only focus on relevant sentences. 
}
\label{fig:norm_comparison}
\end{figure}

\subsection{Behavior Difference}
Similar to \Cref{prop:diff}, to illustrate our insights, we will consider a setting where the smaller model learns useful features for the evaluation task while the larger model covers extra features. 
That is, for evaluation, we uniformly draw a task $\t_\tau = (i_\tau,j_\tau)$ from $S_1$, and then draw $M$ samples to form the evaluation prompt in the same way as during pretraining.   
To present our theorem (proved in \Cref{app:thm:parity_decouple} using \Cref{theorem:opt_parity}), we introduce some notations. 
Let 
\begin{align*}
    & \D_1 =  \big[\diag(\V^{*,(1)}), \dots, \diag(\V^{*,(\nu_1)}), \diag(\V^{*,(\nu_2+1)}), \\ 
    &\dots, \diag(\V^{*,(\nu_2+\nu_1+1)}), \diag(\V^{*,(2\nu_2+2)})\big] \in \R^{d \times 2(\nu_1 +1)}\\
    & \D_2 =  \left[\diag(\V^{*,(1)}), \dots, \diag(\V^{*,(2\nu_2+2)})\right] \in \R^{d \times 2(\nu_2 +1)},
\end{align*}
where for any $i \in [2(\nu_2 +1)]$, $\V^{*,(i)}$ is defined in \Cref{theorem:opt_parity}. 
Let $\hat{\phi}_{\tau, q} \in \R^d$ satisfy $\hat{\phi}_{\tau, q, i_\tau} = {\phi}_{\tau, q, i_\tau}, \hat{\phi}_{\tau, q, j_\tau} = {\phi}_{\tau, q, j_\tau}$ and all other entries being zero. 
For a matrix $\Z$ and a vector $\v$, let $P_\Z$ denote the projection of $\v$ to the space of $\Z$, i.e., $P_\Z(\v) = \Z (\Z^\top \Z)^{-1} \Z^\top  \v$. 

\todo[inline,color=gray!10]{
{\normalsize
\vspace{-0.5em}
\begin{restatable}[Behavior difference for parity]{thm}{parityrobust}\label{theorem:parity_decouple}
Assume the same condition as \Cref{theorem:opt_parity}. For $j\in \{1,2\}$,
Let $\theta_j$ denote the parameters of $g^*_j$. For $l \in [M]$, let $\xi_l$ be uniformly drawn from $\{\pm 1\}^{d}$, and $\Xi = {\sum_{l \in [M]} \xi_l \over M}$. 
Then, for any $\delta \in (0,1)$, with probability at least $1-\delta$ over the randomness of test data, we have
\begin{align}
    & g^*_j(\X_\tau, \y_\tau, \x_{\tau,q}) 
    = ~ h(\theta_j, 2 \gamma  \hat{\phi}_{\tau, q}  + P_{\D_j}(\Xi)) + \epsilon_j  \nonumber\\
    := & \sum_{i \in [m]} \a^*_i\act\left[ \diag\left(\V^{*,(i)} \right)^\top \left(2 \gamma  \hat{\phi}_{\tau, q}  + P_{\D_j}(\Xi)\right)   \right] {+} \epsilon_j \nonumber
\end{align}
where $\epsilon_j = O\left(\sqrt{{\nu_j\over M} \log{1\over \delta}}\right)$ and we have 
\begin{itemize}
    \item $2 \gamma  \hat{\phi}_{\tau, q}$ is the signal useful for prediction: $0 = $ $\ell(y_q \cdot h(\theta_1, 2 \gamma  \hat{\phi}_{\tau, q} )) =$ $\ell(y_q \cdot h(\theta_2, 2 \gamma  \hat{\phi}_{\tau, q} ))$. 
    \item  $P_{\D_1}(\Xi))$ and $P_{\D_2}(\Xi))$ is noise not related to labels, and ${\E[\|P_{\D_1}(\Xi))\|_2^2] \over \E[\|P_{\D_2}(\Xi))\|_2^2]} = {\nu_1+1 \over \nu_2+1}$.
\end{itemize}
\end{restatable}
}}

\textbf{Implications.}
\Cref{theorem:parity_decouple} shows that during evaluation, we can decompose the input into two parts: signal and noise. 
Both the larger model and smaller model can capture the signal part well. 
However, the smaller model has a much smaller influence from noise than the larger model, i.e., the ratio is ${\nu_1+1 \over \nu_2+1}$.
The reason is that smaller models emphasize important hidden features while larger ones cover more hidden features, and thus, smaller models are more robust to noise while larger ones are easily distracted, leading to different ICL behaviors.
This again sheds light on where transformers pay attention to and how that affects ICL.

\begin{remark}
Here, we provide a detailed intuition about \Cref{theorem:parity_decouple}.
$\Xi$ is the input noise. When we only care about the noise part, we can rewrite the smaller model as $g_1 = h(\theta_1, P_{D_1}(\Xi))$, and the larger model as $g_2 = h(\theta_2, P_{D_2}(\Xi))$, where they share the same $h$ function. 
Our conclusion says that $E[\|P_{D_1}(\Xi)\|_2^2] / E[\|P_{D_2}(\Xi)\|_2^2] = (\nu_1 + 1) / (\nu_2 + 1)$, which means the smaller model's ``effect'' input noise is smaller than the larger model's ``effect'' input noise. 
Although their original input noise is the same, as the smaller model only focuses on limited features, the smaller model will ignore part of the noise, and the ``effect'' input noise is small. However, the larger model is the opposite.
\end{remark}

\section{Experiments}\label{sec:exp}
Brilliant recent work~\cite{wei2023larger} runs intensive and thorough experiments to show that larger language models do in-context learning differently.
Following their idea, we conduct similar experiments on binary classification datasets, which is consistent with our problem setting in the parity case, to support our theory statements.

\textbf{Experimental setup.} Following the experimental protocols in~\citet{wei2023larger,min2022rethinking}, we conduct experiments on five prevalent NLP tasks, leveraging datasets from GLUE \cite{wang-etal-2018-glue} tasks and Subj \cite{conneau-kiela-2018-senteval}. Our experiments utilize various sizes of the Llama model families \cite{touvron2023llama,touvron2023llama2}: 3B, 7B, 13B, 70B. We follow the prior work on in-context learning \cite{wei2023larger} and use $M = 16$ in-context exemplars. We aim to assess the models' ability to use inherent semantic biases from pretraining when facing in-context examples. As part of this experiment, we introduce noise by inverting an escalating percentage of in-context example labels. To illustrate, a 100\% label inversion for the SST-2 dataset implies that every ``positive'' exemplar is now labeled ``negative''. Note that while we manipulate the in-context example labels, the evaluation sample labels remain consistent. We use the same templates as~\cite{min2021metaicl}, a sample evaluation for SST-2 when $M=2$:
{\small \begin{verbatim}
sentence: show us a good time  
The answer is positive. 
 
sentence: as dumb and cheesy  
The answer is negative.
   
sentence: it 's a charming and often 
affecting journey
The answer is
\end{verbatim}
}

\subsection{Behavior Difference}
\Cref{fig:nlp_in_context} shows the result of model performance (chat/with instruct turning) across all datasets with respect to the proportion of labels that are flipped. When 0\% label flips, we observe that larger language models have better in-context abilities. On the other hand, the performance decrease facing noise is more significant for larger models. As the percentage of label alterations increases, which can be viewed as increasing label noise $\sigma^2$, the performance of small models remains flat and seldom is worse than random guessing while large models are easily affected by the noise, as predicted by our analysis.
These results indicate that large models can override their pretraining biases in-context input-label correlations, while small models may not and are more robust to noise. 
This observation aligns with the findings in \citet{wei2023larger} and our analysis.

We can see a similar or even stronger phenomenon in \Cref{fig:nlp_in_context_nonchat}: larger models are more easily affected by noise (flipped labels) and override pretrained biases than smaller models for the original/without instruct turning version (see the ``Average'' sub-figure). 
On the one hand, we conclude that both large base models and large chat models suffer from ICL robustness issues.
On the other hand, this is also consistent with recent work suggesting that instruction tuning will impair LLM's in-context learning capability.

\subsection{Ablation Study}
To further verify our analysis, we provide an ablation study. We concatenate an irrelevant sentence from GSM-IC \citep{shi2023large} to an input-label pair sentence from SST-2 in GLUE dataset. We use ``correct'' to denote the original label and ``wrong'' to denote the flipped label. Then, we measure the magnitude of correlation between label-input, by computing the norm of the last row of attention maps across all heads in the final layer. We do this between ``correct''/``wrong'' labels and the original/irrelevant inserted sentences.  
\Cref{fig:norm_comparison} shows the results on 100 evaluation prompts; for example, the subfigure Correct+Relevant shows the correlation magnitude between the ``correct'' label and the original input sentence in each prompt. 
The results show that the small model Llama 2-13b mainly focuses on the relevant part (original input) and may ignore the irrelevant sentence, while the large model Llama 2-70b focuses on both sentences. This well aligns with our analysis. 

\section{More Discussions about Noise}
There are three kinds of noise covered in our analysis: 

{\bf Pretraining noise.} We can see it as toxic or harmful pretraining data on the website (noisy training data). The model will learn these features and patterns. It is covered by $\xi$ in the linear regression case and $S_2$ in the parity case.
 
{\bf Input noise during inference.} We can see it as natural noise as the user's wrong spelling or biased sampling. It is a finite sampling error as $x$ drawn from the Gaussian distribution for the linear regression case and a finite sampling error as $x$ drawn from a uniform distribution for the parity case.

{\bf Label noise during inference.} We can see it as adversarial examples, or misleading instructions, e.g., deliberately letting a model generate a wrong fact conclusion or harmful solution, e.g., poison making. It is $\sigma$ in the linear regression case and $S_2$ in the parity case.

For pretraining noise, it will induce the model to learn noisy or harmful features. During inference, for input noise and label noise, the larger model will pay additional attention to these noisy or harmful features in the input and label pair, i.e., $y \cdot x$, so that the input and label noise may cause a large perturbation in the final results. If there is no pretraining noise, then the larger model will have as good robustness as the smaller model. Also, if there is no input and label noise, the larger model will have as good robustness as the smaller model. The robustness gap only happens when both pretraining noise and inference noise exist simultaneously.
\section{Conclusion}
In this work, we answered our research question: why do larger language models do in-context learning differently? Our theoretical study showed that smaller models emphasize important hidden features while larger ones cover more hidden features, and thus the former are more robust to noise while the latter are more easily distracted, leading to different behaviors during in-context learning. Our empirical results provided positive support for the theoretical analysis. Our findings can help improve understanding of LLMs and ICL, and better training and application of these models.

\section*{Acknowledgements}
The work is partially supported by Air Force Grant FA9550-18-1-0166, the National Science Foundation (NSF) Grants 2008559-IIS, 2023239-DMS, and CCF-2046710.

\section*{Impact Statement}
Our work aims to improve the understanding of the in-context learning mechanism and to inspire efficient and safe use of ICL. Our paper is purely theoretical and empirical in nature and thus we foresee no immediate negative ethical impact.
We hope our work will inspire effective algorithm design and promote a better understanding of large language model learning mechanisms.


\bibliography{ref}
\bibliographystyle{icml2024}

\newpage
\appendix
\onecolumn

\begin{center}
	\textbf{\LARGE Appendix }
\end{center}

\section{Limitations}
We study and understand an interesting phenomenon of in-context learning: smaller models are more robust to noise, while larger ones are more easily distracted, leading to different ICL behaviors. Although we study two stylized settings and give the closed-form solution, our analysis cannot extend to real Transformers easily due to the high non-convex function and complicated design of multiple-layer Transformers. Also, our work does not study optimization trajectory, which we leave as future work. On the other hand, we use simple binary classification real-world datasets to verify our analysis, which still has a gap for the practical user using the LLM scenario.

\section{Deferred Proof for Linear Regression}

\subsection{Proof of \Cref{thm:low_opt}}\label{app:thm:low_opt}
Here, we provide the proof of \Cref{thm:low_opt}. 

\lowopt*
\begin{proof}[Proof of \Cref{thm:low_opt}]
Note that, 
\begin{align*}
    \argmin_{\U \in \R^{d\times d}, \rank(\U) \le r, u \in \R} \tilde{\ell}(\U, u) = & \argmin_{\U \in \R^{d\times d}, \rank(\U) \le r, u \in \R} \tilde{\ell}(\U, u) -  \min_{\U \in \R^{d\times d}, u \in \R} \tilde{\ell}(\U, u)\\
    = &  \argmin_{\U \in \R^{d\times d}, \rank(\U) \le r, u \in \R} \left(\tilde{\ell}(\U, u) -  \min_{\U \in \R^{d\times d}, u \in \R} \tilde{\ell}(\U, u)\right).
\end{align*}
Thus, we may consider \Cref{eq:min_gap} in \Cref{lem:min} only. 
On the other hand, we have 
\begin{align*}
    \Gamma = &  \left(1+{1\over N}\right)\Lambda + {1\over N} \tr(\Lambda) I_{d \times d} \\
    = & \left(1+{1\over N}\right)\Q \eigen \Q^\top + {1\over N} \tr(\eigen) \Q I_{d \times d} \Q^\top\\
    = & \Q \left(\left(1+{1\over N}\right) \eigen + {1\over N} \tr(\eigen) I_{d \times d} \right) \Q^\top. 
\end{align*}
We denote $\eigen' = \left(1+{1\over N}\right) \eigen + {1\over N} \tr(\eigen) I_{d \times d}$. We can see $\Lambda^{1 \over 2} = \Q {\eigen}^{ 1\over2}\Q^\top$, $\Gamma^{1 \over 2} = \Q {\eigen'}^{ 1\over2}\Q^\top$, and $\Gamma^{-1} = \Q {\eigen'}^{-1}\Q^\top$. We denote $\V = u\Q^\top \U\Q$. Since $\Gamma$ and $\Lambda$ are commutable and the Frobenius norm ($F$-norm) of a matrix does not change after multiplying it by an orthonormal matrix, we have \Cref{eq:min_gap} as
\begin{align*}
    \tilde{\ell}(\U, u) -  \min_{\U \in \R^{d\times d}, u \in \R} \tilde{\ell}(\U, u) 
    = & {1\over 2} \left\|\Gamma^{1\over 2} \left(u \Lambda^{1\over 2}\U \Lambda^{1\over 2} -\Lambda\Gamma^{-1} \right) \right\|_F^2 \\
    = & {1\over 2} \left\|\Gamma^{1\over 2}\Lambda^{1\over 2} \left(u \U  - \Gamma^{-1} \right)\Lambda^{1\over 2} \right\|_F^2 \\
    = & {1\over 2} \left\|{\eigen'}^{ 1\over2} {\eigen}^{ 1\over2} \left({\V}  -  {\eigen'}^{-1} \right) {\eigen}^{ 1\over2} \right\|_F^2.
\end{align*}
As $\W^{KQ}$ is a matrix whose rank is at most $r$, we have $\V$ is also at most rank $r$. Then, we denote $\V^* = \argmin_{\V \in \R^{d\times d}, \rank(\V) \le r} \left\|{\eigen'}^{ 1\over2} {\eigen}^{ 1\over2} \left({\V}  -  {\eigen'}^{-1} \right) {\eigen}^{ 1\over2} \right\|_F^2$. We can see that $\V^*$ is a diagonal matrix. Denote $\eigen' = \diag([\lambda_1', \dots, \lambda_d'])$ and $\V^* = \diag([v^*_1, \dots, v^*_d])$. Then, we have 
\begin{align}
    & \left\|{\eigen'}^{ 1\over2} {\eigen}^{ 1\over2} \left({\V}  -  {\eigen'}^{-1} \right) {\eigen}^{ 1\over2} \right\|_F^2 \\
    = & \sum_{i=1}^d \left({\lambda_i'}^{1\over 2}\lambda_i\left(v^*_i - {1\over {\lambda_i'}}\right) \right)^2\\
    = & \sum_{i=1}^d  \left(\left(1+ {1\over N} \right)\lambda_i + {\tr(\eigen)\over N}\right)\lambda_i^2\left(v^*_i - {1\over {\left(1+ {1\over N} \right)\lambda_i + {\tr(\eigen)\over N}}}\right)^2. \label{eq:loss_detail}
\end{align}
As $\V^*$ is the minimum rank $r$ solution, we have that $v^*_i \ge 0$ for any $i \in [d]$ and if $v^*_i > 0$, we have $v^*_i = {1\over {\left(1+ {1\over N} \right)\lambda_i + {\tr(\eigen)\over N}}}$. Denote $g(x) = \left(\left(1+ {1\over N} \right)x + {\tr(\eigen)\over N}\right)x^2\left({1\over {\left(1+ {1\over N} \right)x + {\tr(\eigen)\over N}}}\right)^2 = x^2\left({1\over {\left(1+ {1\over N} \right)x + {\tr(\eigen)\over N}}}\right)$. It is easy to see that $g(x)$ is an increasing function on $[0,\infty)$. Now, we use contradiction to show that $\V^*$ only has non-zero entries in the first $r$ diagonal entries. Suppose $i>r$, such that $v^*_i > 0$, then we must have $j \le r$ such that $v^*_j = 0$ as $\V^*$ is a rank $r$ solution. We find that if we set $v^*_i = 0, v^*_j = {1\over {\left(1+ {1\over N} \right)\lambda_j + {\tr(\eigen)\over N}}}$ and all other values remain the same, \Cref{eq:loss_detail} will strictly decrease as $g(x)$ is an increasing function on $[0,\infty)$. Thus, here is a contradiction. We finish the proof by $\V^* = u\Q^\top \U^* \Q$. 
\end{proof}

\subsection{Behavior Difference}\label{app:thm:mse}  
\mse*
\begin{proof}[Proof of \Cref{thm:mse}]
By \Cref{thm:low_opt}, w.l.o.g, letting $c=1$, the optimal rank-$r$ solution $f_{\textup{LSA},\theta}$ satisfies $\theta = (\W^{PV}, \W^{KQ})$, and
\begin{align*}
& \W^{*PV} = \begin{pmatrix}
0_{d\times d} & 0_d \\
0_d^\top & 1 
\end{pmatrix},  \W^{*KQ} = \begin{pmatrix}
\U^* & 0_d \\
0_d^\top & 0 
\end{pmatrix},
\end{align*}
where $\U^* = \Q \V^* \Q^\top$. 

We can see that $\U^*$ and $\Lambda$ commute. 
Denote $\widehat{\Lambda}:= {1\over M}\sum_{i=1}^M \x_i \x_i^\top$. Note that we have 
\begin{align*}
    \widehat{y}_{q} = & f_{\textup{LSA},\theta}(\widehat{\bE}) \\
    = & \begin{pmatrix}
0_{d\times d} & 0_d \\
0_d^\top & 1 
\end{pmatrix}\left( {\widehat{\bE} \widehat{\bE}^\top \over M} \right) 
    \begin{pmatrix}
\U^* & 0_d \\
0_d^\top & 0 
\end{pmatrix} \x_{q} \\
= & \begin{pmatrix}
0_{d\times d} & 0_d \\
0_d^\top & 1 
\end{pmatrix} 
\begin{pmatrix}
{1\over M} \left(\x_q \x_q^\top + \sum_{i=1}^M \x_i \x_i^\top \right) & {1\over M} \left(\sum_{i=1}^M \x_i \x_i^\top \w + \sum_{i=1}^M \epsilon_i \x_i \right)  \\
{1\over M} \left(\sum_{i=1}^M \w^\top \x_i \x_i^\top + \sum_{i=1}^M \epsilon_i \x_i^\top \right) & {1\over M} \sum_{i=1}^M (\w^\top \x_i + \epsilon_i)^2 
\end{pmatrix} \nonumber
\\
& \cdot \begin{pmatrix}
\U^* & 0_d \\
0_d^\top & 0 
\end{pmatrix} \x_{q}  \\
= & \left(\w^\top \widehat{\Lambda} + {1\over M} \sum_{i=1}^M \epsilon_i \x_i^\top \right) \U^* \x_{q}. 
\end{align*}
Then, we have
\begin{align*}
    & \E_{\x_1,\epsilon_1,\dots,\x_M,\epsilon_M,\x_q}\left( \widehat{y}_{q} - \inner{\w,\x_q}\right)^2\\
    = & \E_{\x_1,\epsilon_1,\dots,\x_M,\epsilon_M,\x_q} \left(\w^\top \widehat{\Lambda} \U^* \x_{q} + {1\over M}\sum_{i=1}^M \epsilon_i \x_i^\top \U^* \x_{q}   - {\w^\top \x_q}\right)^2
    \\
    = & \underbrace{\E \left[\left(\w^\top \widehat{\Lambda} \U^* \x_{q}- {\w^\top \x_q}\right)^2 \right]}_{\text{(I)}} + \underbrace{\E \left[\left({1\over M}\sum_{i=1}^M \epsilon_i \x_i^\top \U^* \x_{q}\right)^2\right]}_{\text{(II)}},
\end{align*}
where the last equality is due to i.i.d. of $\epsilon_i$. We see that the label noise can only have an effect in the second term. For the term (I) we have,
\begin{align*}
    \text{(I)} = & \E \left[\left(\w^\top \widehat{\Lambda} \U^* \x_{q} - \w^\top \Lambda \U^* \x_{q} + \w^\top \Lambda \U^* \x_{q} -  {\w^\top \x_q}\right)^2 \right] \\
    = & \underbrace{\E \left[\left(\w^\top \widehat{\Lambda} \U^* \x_{q} - \w^\top \Lambda \U^* \x_{q}\right)^2 \right]}_{\text{(III)}} +  \underbrace{\E \left[\left(\w^\top \Lambda \U^* \x_{q} -  {\w^\top \x_q}\right)^2 \right]}_{\text{(IV)}},
\end{align*}
where the last equality is due to $\E[\widehat{\Lambda}] = \Lambda$ and $\widehat{\Lambda}$ is independent with $\x_{q}$. Note the fact that $\U^*$ and $\Lambda$ commute. For the (III) term, we have
\begin{align*}
    \text{(III)} = & \E \left[\E \left[\left(\w^\top \widehat{\Lambda} \U^* \x_{q}\right)^2 + \left( \w^\top \Lambda \U^* \x_{q}\right)^2 - 2\left(\w^\top \widehat{\Lambda} \U^* \x_{q}\right)\left( \w^\top \Lambda \U^* \x_{q}\right) \right] \middle| \x_q\right] \\
    = & \E \left[\left(\w^\top \widehat{\Lambda} \U^* \x_{q}\right)^2 - \left( \w^\top \Lambda \U^* \x_{q}\right)^2 \right].
\end{align*}
By the property of trace, we have,
\begin{align*}
    \text{(III)} = &  \E \left[ \tr\left(\widehat{\Lambda}\w \w^\top \widehat{\Lambda}   (\U^*)^2 \Lambda\right)\right] - \|\w\|_{(\U^*)^2\Lambda^3}^2\\
    = &  \E \left[ {1\over M^2}\tr\left( \left(\sum_{i=1}^M \x_i \x_i^\top\right) \w \w^\top \left(\sum_{i=1}^M \x_i \x_i^\top\right)   (\U^*)^2 \Lambda\right)\right] - \|\w\|_{(\U^*)^2\Lambda^3}^2\\
    = &  \E \left[ {M-1\over M}\tr\left( \Lambda \w \w^\top \Lambda   (\U^*)^2 \Lambda\right) + {1\over M}\tr\left( \x_1 \x_1^\top \w \w^\top \x_1 \x_1^\top   (\U^*)^2 \Lambda\right)\right] - \|\w\|_{(\U^*)^2\Lambda^3}^2\\
    = &  -{1\over M} \|\w\|_{(\U^*)^2 \Lambda^3}^2 + {1\over M} \E \left[ \tr\left( \x_1 \x_1^\top \w \w^\top \x_1 \x_1^\top   (\U^*)^2 \Lambda\right)\right] 
    \\
    = &  -{1\over M} \|\w\|_{(\U^*)^2 \Lambda^3}^2 + {1\over M} \E \left[ \tr\left( \left(\|\w\|_\Lambda^2 \Lambda + 2 \Lambda \w^\top \w\Lambda\right)   (\U^*)^2 \Lambda\right)\right] \\
    = &  {1\over M} \|\w\|_{(\U^*)^2 \Lambda^3}^2 + {1\over M} \|\w\|_\Lambda^2  \tr\left(    (\U^*)^2 \Lambda^2\right),
\end{align*}
where the third last equality is by \Cref{lem:Isserlis}. Furthermore, injecting $\w = \Q(\s + \xi)$, as $\xi^\top \V^*$ is a zero vector, we have
\begin{align*}
    \text{(III)} & =  {1\over M} \|\s + \xi\|_{(\V^*)^2 \eigen^3}^2 + {1\over M} \|\s + \xi\|_\eigen^2  \tr\left(    (\V^*)^2 \eigen^2\right) \\
    & =  {1\over M} \|\s \|_{(\V^*)^2 \eigen^3}^2 + {1\over M} \|\s + \xi\|_\eigen^2  \tr\left(    (\V^*)^2 \eigen^2\right).
\end{align*}
Similarly, for the term (IV), we have
\begin{align*}
    \text{(IV)} = & \E \left[\left((\s+\xi)^\top \Q^\top \Lambda \U^* \x_{q} -  {(\s+\xi)^\top \Q^\top \x_q}\right)^2 \right] \\
    = & \E \left[\left(\s^\top \eigen \V^* \Q^\top \x_{q} -  {\s^\top \Q^\top \x_q} -  {\xi^\top \Q^\top \x_q}\right)^2 \right] \\
    = & \s^\top (\V^*)^2 \eigen^3 \s + \s^\top  \eigen \s + \xi^\top  \eigen \xi - 2 \s^\top \V^* \eigen^2 \s\\
    = & \xi^\top \eigen \xi + \sum_{i \in[r]} \s_i^2 \lambda_i\left(\lambda_i^2 (v^*_i)^2 - 2 \lambda_i v^*_i + 1 \right)\\
    = & \|\xi\|_\eigen^2 + \sum_{i \in[r]} \s_i^2 \lambda_i\left( \lambda_i v^*_i - 1 \right)^2, 
\end{align*}
where the third equality is due to $\s^\top \A \xi = 0$ for any diagonal matrix $\A \in \R^{d \times d}$.   

Now, we analyze the label noise term. By $\U^*$ and $\Lambda$ being commutable, for the term (II), we have
\begin{align*}
    \text{(II)} = &  {\sigma^2\over M^2} \E \left[\left(\sum_{i=1}^M \x_i^\top \U^* \x_{q}\right)^2\right]\\
    = &  {\sigma^2\over M^2} \E \left[\tr\left(\left(\sum_{i=1}^M \x_i\right)^\top \U^* \Lambda \U^* \left(\sum_{i=1}^M \x_i\right)\right)\right]\\
    = &  {\sigma^2\over M} \E \left[\tr\left(\x_1^\top \U^* \Lambda \U^* \x_1\right)\right]
    \\
    = &  {\sigma^2\over M} \tr\left(( \V^*)^2 \eigen^2 \right),
\end{align*}
where all cross terms vanish in the second equality.  
We conclude by combining four terms. 
\end{proof}

\diff*
\begin{proof}[Proof of \Cref{prop:diff}]
Let $\V^* = \diag([v^*_1, \dots, v^*_d])$ satisfying for any $i \le {r}, v^*_i = {N \over {\left(N+1 \right)\lambda_i + {\tr(\eigen)}}}$ and for any $i > {r}, v^*_i = 0$. Let ${\V'}^* = \diag([{v'}^*_1, \dots, {v'}^*_d])$ be satisfied for any $i \le {r'}, {v'}^*_i = {N \over {\left(N+1 \right)\lambda_i + {\tr(\eigen)}}}$ and for any $i > {r'}, {v'}^*_i = 0$. Note that $\V^*$ is a truncated diagonal matrix of ${\V'}^*$.  By \Cref{thm:low_opt} and \Cref{thm:mse}, we have
\begin{align*}
    \mathcal{L}(f_2;\widehat{\bE}) - \mathcal{L}(f_1;\widehat{\bE})
    = & \left({1\over M} \|\s \|_{({\V'}^*)^2 \eigen^3}^2 + {1\over M} \left(\|\s \|_\eigen^2 + \sigma^2\right)   \tr\left(    ({\V'}^*)^2 \eigen^2\right) + \sum_{i \in[{r'}]} \s_i^2 \lambda_i\left( \lambda_i {v'}^*_i - 1 \right)^2 \right) \\
    & - \left({1\over M} \|\s \|_{(\V^*)^2 \eigen^3}^2 + {1\over M} \left(\|\s \|_\eigen^2 + \sigma^2\right)   \tr\left(    (\V^*)^2 \eigen^2\right) + \sum_{i \in[{r}]} \s_i^2 \lambda_i\left( \lambda_i v^*_i - 1 \right)^2 \right)\\
    = & {1\over M} \left(\|\s \|_\eigen^2 + \sigma^2\right)   \left(\tr\left(    ({\V'}^*)^2 \eigen^2\right) - \tr\left((\V^*)^2 \eigen^2\right) \right)   \\
    = & {1\over M} \left(\|\s \|_\eigen^2 + \sigma^2\right)   \left(\sum_{i = {r}+1}^{{r'}} \left({N \lambda_i \over {\left(N+1 \right)\lambda_i + {\tr(\eigen)}}}\right)^2\right).
\end{align*}
\end{proof}

\subsection{Auxiliary Lemma}
\Cref{lem:min} provides the structure of the quadratic form of our MSE loss. 
\begin{lemma}[Corollary A.2 in~\citet{zhang2023trained}]\label{lem:min}
The loss function $\tilde{\ell}$ in \Cref{lem:loss_simple} satisfies
\begin{align*}
    \min_{\U \in \R^{d\times d}, u \in \R} \tilde{\ell}(\U, u) = -{1\over 2}\tr[\Lambda^2\Gamma^{-1}],
\end{align*}
where $\U = c\Gamma^{-1}, u={1\over c}$ for any non-zero constant $c$ are minimum solution. We also have
\begin{align}
   \tilde{\ell}(\U, u) -  \min_{\U \in \R^{d\times d}, u \in \R} \tilde{\ell}(\U, u) = {1\over 2} \left\|\Gamma^{1\over 2} \left(u \Lambda^{1\over 2}\U \Lambda^{1\over 2} -\Lambda\Gamma^{-1} \right) \right\|_F^2. \label{eq:min_gap}
\end{align}
\end{lemma}

\begin{lemma}\label{lem:Isserlis}
    Let $\x \sim \cN(0, \Lambda), \epsilon\sim \cN(0, \sigma^2)$ and $y = \inner{\w,\x} + \epsilon$, where $\w \in \R^d$ is a fixed vector. Then we have 
    \begin{align*}
        \E\left[y^2 \x\x^\top\right] = & \sigma^2\Lambda + \|\w\|_\Lambda^2 \Lambda + 2 \Lambda \w^\top \w\Lambda,\\
        \E(y\x) \E(y\x)^\top = & \Lambda^\top \w \w^\top \Lambda, \\
        \E\left[(y\x - \E(y\x) ) (y\x - \E(y\x) ) ^\top\right] = &  \sigma^2\Lambda + \|\w\|_\Lambda^2 \Lambda + \Lambda \w^\top \w\Lambda.
    \end{align*} 
\end{lemma}
\begin{proof}[Proof of \Cref{lem:Isserlis}]
As $y$ is a zero mean Gaussian, by Isserlis' theorem~\cite{wick1950evaluation,michalowicz2009isserlis}, for any $i, j \in [d]$ we have 
\begin{align*}
   \E[y^2 \x_i \x_j] = & \E[y^2]\E[\x_i \x_j] + 2\E[y \x_i]\E[y \x_j]\\
   = & \left(\sigma^2 + \w^\top \Lambda \w \right)\Lambda_{i,j} + 2 \Lambda_{i}^\top \w \w^\top \Lambda_{j}.
\end{align*}
Thus, we have $\E\left[y^2 \x\x^\top\right] =  \left(\sigma^2 + \w^\top \Lambda \w \right)\Lambda + 2 \Lambda \w^\top \w\Lambda $. Similarly, we also have $\E(y\x) \E(y\x)^\top = \Lambda^\top \w \w^\top \Lambda$. Thus, we have 
\begin{align*}
     & \E\left[(y\x - \E(y\x) ) (y\x - \E(y\x) ) ^\top\right] \\
     = & \E\left[y^2 \x\x^\top - y\x \E(y\x)^\top - \E(y\x) y\x^\top +  \E(y\x) \E(y\x)^\top\right] \\
     = & \E\left[y^2 \x\x^\top\right] -  \E(y\x) \E(y\x)^\top
     \\
     = & \left(\sigma^2 + \w^\top \Lambda \w \right)\Lambda + \Lambda \w^\top \w\Lambda.
\end{align*}
\end{proof}
\section{Deferred Proof for Parity Classification}

\subsection{Proof of \Cref{theorem:opt_parity}}\label{app:thm:opt_parity}
Here, we provide the proof of \Cref{theorem:opt_parity}.
\optparity*
\begin{proof}[Proof of \Cref{theorem:opt_parity}]
Recall $\t_\tau = (i_\tau,j_\tau)$. Let $\z_{\tau} \in \R^d$ satisfy $\z_{\tau, i_\tau} = \z_{\tau, j_\tau} = 2\gamma$ and all other entries are zero. Denote $\V^{(i)} = \G^\top \W^{(i)} \G$. Notice that $\|\W^{(i)}\|_F^2 = \|\V^{(i)}\|_F^2$. Thus, we denote  $\V^{*, (i)} = \G^\top \W^{*, (i)} \G$.
Then, we have
\begin{align*}
    & \E_\tau\left[\ell\left( y_{\tau,q} \cdot g(\X_\tau, \y_\tau, \x_{\tau,q})\right) \right] \\
    = & \E_\tau\left[\ell\left( y_{\tau,q} \left(\sum_{i \in [m]} \a_i\act\left[ {\y_\tau^\top \X_\tau \over N} \W^{(i)} \x_{\tau, q} \right] \right) \right)\right]\\
    = & \E_\tau\left[\ell\left( y_{\tau,q} \left(\sum_{i \in [m]} \a_i\act\left[ \z_\tau^\top \V^{(i)} \phi_{\tau, q} \right]   \right) \right)\right]\\
    = & \E_\tau\left[\ell\left( y_{\tau,q} \left(\sum_{i \in [m]} \a_i\act\left[ 2\gamma (\V_{i_\tau,:}^{(i)} +\V_{j_\tau,:}^{(i)})\phi_{\tau, q} \right]   \right) \right)\right].
\end{align*} 
We can see that for any $i \in [m],$ $|\a^*_i| = 1$ and $\V_{j,l}^{*, (i)} = 0$ when $j\neq l$. As ReLU is a homogeneous function, we have
\begin{align*}
    & \E_\tau\left[\ell\left( y_{\tau,q} \cdot g^*(\X_\tau, \y_\tau, \x_{\tau,q})\right) \right]\\
    = & \underbrace{(1-p_\cT) \E\left[\ell\left( 2\gamma \phi_{\tau, q, i_\tau}\phi_{\tau, q, j_\tau} \left(\sum_{i \in [m]} \a^*_i\act\left[ \V_{i_\tau,i_\tau}^{*,(i)}  \phi_{\tau, q, i_\tau}+\V_{j_\tau,j_\tau}^{*,(i)}  \phi_{\tau, q, j_\tau} \right]   \right) \right)\middle| \t_\tau \in S_1  \right]}_{\text{(I)}}\\
    & \quad~~ + \underbrace{p_\cT \E\left[\ell\left( 2\gamma \phi_{\tau, q, i_\tau}\phi_{\tau, q, j_\tau} \left(\sum_{i \in [m]} \a^*_i\act\left[ \V_{i_\tau,i_\tau}^{*,(i)}  \phi_{\tau, q, i_\tau} + \V_{j_\tau,j_\tau}^{*,(i)}  \phi_{\tau, q, j_\tau} \right]     \right) \right)\middle| \t_\tau \in S_2 \right]}_{\text{(II)}}.
\end{align*}
We have
\begin{align*}
\text{(I)} = & (1-p_\cT) \cdot \Bigg\{({1\over 4} + \gamma) \E\left[\ell\left( 2\gamma \left(\sum_{i \in [m]} \a^*_i\act\left[ \V_{i_\tau,i_\tau}^{*,(i)}  +\V_{j_\tau,j_\tau}^{*,(i)} \right]   \right) \right)\middle| \t_\tau \in S_1\right] \\
& \quad\quad\quad\quad\quad +  {1\over 4} \E\left[\ell\left( -2\gamma \left(\sum_{i \in [m]} \a^*_i\act\left[ \V_{i_\tau,i_\tau}^{*,(i)}  -\V_{j_\tau,j_\tau}^{*,(i)} \right]   \right) \right)\middle| \t_\tau \in S_1\right] 
\\
& \quad\quad\quad\quad\quad +  ({1\over 4} - \gamma) \E\left[\ell\left( 2\gamma \left(\sum_{i \in [m]} \a^*_i\act\left[ -\V_{i_\tau,i_\tau}^{*,(i)}  -\V_{j_\tau,j_\tau}^{*,(i)} \right]   \right) \right)\middle| \t_\tau \in S_1\right] \\
& \quad\quad\quad\quad\quad +  {1\over 4} \E\left[\ell\left( -2\gamma \left(\sum_{i \in [m]} \a^*_i\act\left[ -\V_{i_\tau,i_\tau}^{*,(i)}  +\V_{j_\tau,j_\tau}^{*,(i)} \right]   \right) \right)\middle| \t_\tau \in S_1\right]\Bigg\}. 
\end{align*}
We can get a similar equation for (II). 

\paragraph{We make some definitions to be used.} We define a pattern as $(z_1, \{(i_\tau, z_2),  (j_\tau, z_3)\} )$, where $z_1, z_2, z_3 \in \{\pm 1\}$. We define a pattern is covered by a neuron means there exists $i \in [m]$, such that $\a^*_i = z_1$ and $\sign(\V_{i_\tau,i_\tau}^{*,(i)})=z_2$ and $ \sign(\V_{j_\tau,j_\tau}^{*,(i)})=z_3$. We define a neuron as being positive when its $\a^*_i = +1$ and being negative when its $\a^*_i = -1$. We define a pattern as being positive if $z_1 = +1$ and being negative if $z_1 = -1$.

Then all terms in (I) and (II) can be written as:
\begin{align*}
\alpha \E\left[\ell\left( 2\gamma z_1 \left(\sum_{i \in [m]} \a^*_i\act\left[ z_2 \V_{i_\tau,i_\tau}^{*,(i)}  + z_3 \V_{j_\tau,j_\tau}^{*,(i)} \right]   \right) \right)\right],
\end{align*}
where $\alpha$ is the scalar term. Note that there are total ${k(k-1)\over 2} \times 4$ patterns in (I) and $\left({d(d-1)\over 2}-{k(k-1)\over 2}\right) \times 4$ patterns in (II). 
The loss depends on the weighted sum of non-covered patterns. To have zero loss, we need all patterns to be covered by $m$ neurons, i.e., $(\a^*, \V^{*,(1)}, \dots, \V^{*,(m)})$. 

Note that one neuron at most cover ${d(d-1) \over 2}$ patterns. Also, by $0< p_\cT < { {1\over 4} - \gamma \over {d(d-1) \over 2} ({1\over 4} + \gamma) +  {1\over 4} - \gamma}$, we have 
\begin{align*}
    {d(d-1) \over 2} p_\cT ({1\over 4} + \gamma) < (1-p_\cT)({1\over 4} - \gamma),
\end{align*}
which means the model will only cover all patterns in (I) before covering a pattern in (II) in purpose. 

Now, we show that the minimum number of neurons to cover all patterns in (I) and (II) is $2(\nu_2 + 1)$. 

\paragraph{First, we show that $2(\nu_2 + 1)$ neurons are enough to cover all patterns in (I) and (II). }
For $i\in [\nu_2]$ and $i_\tau \in [d]$, $\V_{i_\tau,i_\tau}^{(i)} = (2\digit(\bin(i_\tau-1), i)-1)/(4\gamma)$ and all non-diagonal entries in $\V^{(i)}$ being zero and $ \a_i = -1$. For $i = \nu_2 +1$ and $i_\tau \in [d]$, $\V_{i_\tau,i_\tau}^{(i)} = -\nu_2/(4\gamma)$ and all non-diagonal entries in $\V^{(i)}$ being zero and $ \a_i = +1$. For $i \in [2(\nu_2 + 1)]\setminus [\nu_2 + 1]$, let $\V^{(i)} = - \V^{(i-\nu_2- 1)}$ and $ \a_i= \a_{i-\nu_2- 1} $. 

We can check that this construction can cover all patterns in (I) and (II) and only needs $2(\nu_2 + 1)$ neurons. $\V^{(\nu_2 + 1)}$ and $\V^{(2(\nu_2 + 1))}$ cover all positive patterns. All other neurons cover all negative patterns. This is because $\bin(i_\tau)$ and $\bin(j_\tau)$ have at least one digit difference. If $\bin(i_\tau)$ and $\bin(j_\tau)$ are different in the $i$-th digit, then $(-1, \{(i_\tau, -1),  (j_\tau, +1)\} )$ and $(-1, \{(i_\tau, +1),  (j_\tau, -1)\} )$ are covered by the $i$-th and $i+\nu_2 + 1$-th neuron. 

We can also check that the scalar ${1\over 4\gamma}$ and ${\nu_2 \over 4\gamma}$ is the optimal value. Note that 
\begin{enumerate}
    \item[(1)] For any negative patterns, the positive neurons will not have a cancellation effect on the negative neurons, i.e., when $y_q = -1$, the positive neurons will never activate. 
    \item[(2)] For each negative neuron, there exist some patterns that are uniquely covered by it. 
    \item[(3)] For any positive patterns, there are at most $\nu_2 - 1$ negative neurons that will have a cancellation effect on the positive neurons, i.e., when $y_q = +1$, these negative neurons will activate simultaneously. Also, we can check that there is a positive pattern such that there are $\nu_2 - 1$ negative neurons that will have a cancellation effect.
    \item[(4)] For two positive neurons, there exist some patterns that are uniquely covered by one of them. 
\end{enumerate}
Due to hinge loss, we can see that ${1\over 4 \gamma}$ is tight for negative neurons as (1) and (2). Similarly, we can also see that ${\nu_2\over 4 \gamma}$ is tight for positive neurons as (3) and (4). 

\paragraph{Second, we prove that we need at least $2(\nu_2 + 1)$ neurons to cover all patterns in (I) and (II).} 
We can see that we need at least 2 positive neurons to cover all positive patterns. Then, we only need to show that $2\nu_2-1$ neurons are not enough to cover all negative patterns. We can prove that all negative patterns are covered equivalent to all numbers from $\{0, 1, \dots, 2^{\nu_2}-1\}$ are encoded by $\left\{\left(\V_{i,i}^{(1)}, \dots, \V_{i,i}^{(\nu_2)}\right) ~\middle|~ i \in [k]\right\}$. Then $2\nu_2-1$ is not enough to do so. 

Therefore, the minimum number of neurons to cover all patterns in (I) and (II) is $2(\nu_2 + 1)$. 

Thus, when $m = 2(\nu_1 +1)$, the optimal solution will cover all patterns in (I) but not all in (II). When $m \ge 2(\nu_2 +1)$, the optimal solution will cover all patterns in (I) and (II). 
We see that $g^*_1$ neurons as the subset of $g^*_2$ neurons, while the only difference is that the scalar of positive neurons is ${\nu_1\over 4 \gamma}$ for $g^*_1$ and ${\nu_2\over 4 \gamma}$ for $g^*_2$. Thus, we finished the proof.
\end{proof}

\subsection{Proof of \Cref{theorem:parity_decouple}}\label{app:thm:parity_decouple}
Here, we provide the proof of \Cref{theorem:parity_decouple}.
\parityrobust*

\begin{proof}[Proof of \Cref{theorem:parity_decouple}]
Let $\Phi^\tau = [\phi_{\tau, 1}, \dots, \phi_{\tau, M} ]^\top \in \R^{M \times d}$. 
Recall $\t_\tau = (i_\tau,j_\tau)$. Let $\z_{\tau} \in \R^d$ satisfy $\z_{\tau, i_\tau} = \z_{\tau, j_\tau} = 2\gamma$ and all other entries are zero. We see $\t_\tau$ as an index set and let $\r_\tau = [d] \setminus \t_\tau $. Then, we have
\begin{align*}
    & g^*_2(\X_\tau, \y_\tau, \x_{\tau,q})\\
    = & \sum_{i \in [m]} \a^*_i\act\left[ {\y_\tau^\top \X_\tau \over M} \W^{*,(i)} \x_{\tau, q} \right] \\
    = & \sum_{i \in [m]} \a^*_i\act\left[ {\y_\tau^\top \Phi^\tau \over M} \V^{*,(i)} \phi_{\tau, q} \right]  \\
    = & \sum_{i \in [m]} \a^*_i\act\left[ {\y_\tau^\top \Phi^\tau_{:,\t_\tau} \over M} \V^{*,(i)}_{\t_\tau,:} \phi_{\tau, q, \t_\tau} + {\y_\tau^\top \Phi^\tau_{:,\r_\tau} \over M} \V^{*,(i)}_{\r_\tau,:} \phi_{\tau, q, \r_\tau}  \right]. 
\end{align*} 
Note that we can absorb the randomness of $\y_\tau, \Phi^\tau_{:,\r_\tau},\phi_{\tau, q, \r_\tau}$ together. 

Let $z_i$ for $i \in [n]$ uniformly draw from $\{-1,+1\}$. By Chernoff bound for binomial distribution (\Cref{lem:Chernoff}), for any $0 < \epsilon < 1$, we have 
\begin{align*}
    \Pr\left(\left|{\sum_{i\in [n]} z_i \over n} \right| \ge \epsilon \right) \le 2 \exp\left(- {\epsilon^2 n \over 6}\right).
\end{align*}

Thus, for any $0 < \delta <1$, with probability at least $1-\delta$ over the randomness of evaluation data, such that
\begin{align*}
    \left|\Xi_{\t_\tau}^\top \diag(\V^{*,(i)}_{\t_\tau,\t_\tau})\right| \le O\left(\sqrt{{1\over M} \log{1\over \delta}}\right).
\end{align*}
Then, for any $0 < \delta <1$, with probability at least $1-\delta$ over the randomness of evaluation data, we have
\begin{align*}
    & g^*_2(\X_\tau, \y_\tau, \x_{\tau,q})\\
    = & \sum_{i \in [m]} \a^*_i\act\left[ {\y_\tau^\top \Phi^\tau_{:,\t_\tau} \over M} \V^{*,(i)}_{\t_\tau,:} \phi_{\tau, q, \t_\tau} + \Xi^\top \diag(\V^{*,(i)}) - \Xi_{\t_\tau}^\top \diag(\V^{*,(i)}_{\t_\tau,\t_\tau})  \right]\\
    = & \sum_{i \in [m]} \a^*_i\act\left[ \z_{\tau}^\top \V^{*,(i)}_{\t_\tau,:} \phi_{\tau, q, \t_\tau} + \Xi^\top \diag(\V^{*,(i)}) - \Xi_{\t_\tau}^\top \diag(\V^{*,(i)}_{\t_\tau,\t_\tau})  \right]\\
    = & \sum_{i \in [m]} \a^*_i\act\left[ 2 \gamma \diag\left(\V^{*,(i)}_{\t_\tau,\t_\tau} \right)^\top \phi_{\tau, q, \t_\tau} + \Xi^\top \diag(\V^{*,(i)}) - \Xi_{\t_\tau}^\top \diag(\V^{*,(i)}_{\t_\tau,\t_\tau})  \right]
    \\
    = & \sum_{i \in [m]} \a^*_i\act\left[ \diag\left(\V^{*,(i)} \right)^\top \left(2 \gamma  \hat{\phi}_{\tau, q}  + \Xi\right)  - \Xi_{\t_\tau}^\top \diag(\V^{*,(i)}_{\t_\tau,\t_\tau})  \right]\\
    = & \sum_{i \in [m]} \a^*_i\act\left[ \diag\left(\V^{*,(i)} \right)^\top \left(2 \gamma  \hat{\phi}_{\tau, q}  + \Xi\right)  + O\left(\sqrt{{1\over M} \log{1\over \delta}}\right)  \right]\\
    = & \sum_{i \in [m]} \a^*_i\act\left[ \diag\left(\V^{*,(i)} \right)^\top \left(2 \gamma  \hat{\phi}_{\tau, q}  + P_{\D_2}(\Xi)\right)  + O\left(\sqrt{{1\over M} \log{1\over \delta}}\right)  \right] \\
    = & ~ h(\theta_2, 2 \gamma  \hat{\phi}_{\tau, q}  + P_{\D_2}(\Xi)) + O\left(\sqrt{{\nu_2\over M} \log{1\over \delta}}\right). 
\end{align*} 
Similarly, we have $g^*_1(\X_\tau, \y_\tau, \x_{\tau,q}) = h(\theta_1, 2 \gamma  \hat{\phi}_{\tau, q}  + P_{\D_1}(\Xi)) + O\left(\sqrt{{\nu_1\over M} \log{1\over \delta}}\right)$. 

As $\t_\tau \in S_1$ and the number of $(\phi_{i_\tau},\phi_{j_\tau})$ being balanced as training, by careful checking,  we can see that $\ell(y_q \cdot h(\theta_1, 2 \gamma  \hat{\phi}_{\tau, q} )) = \ell(y_q \cdot h(\theta_2, 2 \gamma  \hat{\phi}_{\tau, q} ))=0$ and we have $2 \gamma  \hat{\phi}_{\tau, q}$ is the signal part.

On the other hand, we know that all the first half columns in $\D_2$ are orthogonal with each other, and the second half columns in $\D_2$ are opposite to the first half columns. We have the same fact to $\D_1$. As $\Xi$ is a symmetric noise distribution, we have ${\E[\|P_{\D_1}(\Xi))\|_2^2] \over \E[\|P_{\D_2}(\Xi))\|_2^2]} = {\nu_1+1 \over \nu_2+1}$ and we have $P_{\D_1}(\Xi))$ and $P_{\D_2}(\Xi))$ is the noise part.
\end{proof}

\subsection{Auxiliary Lemma}

\begin{lemma}[Chernoff bound for binomial distribution]\label{lem:Chernoff}
Let $Z \sim \Bin(n, p)$ and let $\mu = \E[Z]$. For any $0 < \epsilon < 1$,
we have
\begin{align*}
    \Pr(|Z-\mu| \ge \epsilon\mu) \le 2 \exp\left(- {\epsilon^2 \mu \over 3}\right).
\end{align*}
\end{lemma}


\end{document}